\documentclass[conference]{IEEEtran}
\IEEEoverridecommandlockouts

\usepackage{cite}
\usepackage{amsmath,amssymb,amsfonts}
\usepackage{algorithmic}
\usepackage{graphicx}
\usepackage{textcomp}
\usepackage{xcolor}
\usepackage{balance}

\def\BibTeX{{\rm B\kern-.05em{\sc i\kern-.025em b}\kern-.08em
    T\kern-.1667em\lower.7ex\hbox{E}\kern-.125emX}}

\usepackage[utf8]{inputenc} %
\usepackage[T1]{fontenc}    %
\usepackage{hyperref}       %
\usepackage{url}            %
\usepackage{booktabs}       %
\usepackage{amsfonts}       %
\usepackage{nicefrac}       %
\usepackage{microtype}      %
\usepackage{xcolor}         %

\usepackage{amssymb, amsmath, amsthm}
\usepackage{thmtools,thm-restate,mathtools}
\usepackage{algorithm}
\usepackage{algorithmic}

\usepackage{hyperref}
\usepackage{cleveref}
\Crefformat{figure}{#2Fig.~#1#3}

\usepackage{comment}
\usepackage[normalem]{ulem}
\usepackage{csquotes}
\usepackage{subfigure}

\usepackage{placeins} 

\newtheorem{lemma}{Lemma}
\newtheorem{theorem}{Theorem}

\newtheorem{fact}{Fact}

\newtheorem{definition}{Definition}

\newtheorem{remark}{Remark}

\newcommand{\N}{\mathbb{N}}

\newcommand{\R}{\mathbb{R}}
\newcommand{\M}{\mathcal{M}}
\newcommand{\D}{\mathcal{D}}
\newcommand{\B}{\mathcal{B}}

\newcommand{\gauss}{\mathcal{N}}
\DeclareMathOperator{\mse}{MeanSE}
\DeclareMathOperator{\maxerr}{MaxSE}
\newcommand{\approxdp}{$(\epsilon,\delta)$-DP}
\newcommand{\dpconst}{C_{\epsilon,\delta}}
\DeclareMathOperator*{\expectation}{\mathbb{E}}

\DeclarePairedDelimiter\abs{\lvert}{\rvert}

\renewcommand\vec{\mathbf}

\newcommand{\norm}[1]{\| #1 \|}

\renewcommand{\epsilon}{\varepsilon} 
\usepackage{tikz}
\newcommand\copyrighttext{%
  \footnotesize A version of this work has been accepted for publication in the IEEE Conference on Secure and Trustworthy Machine Learning (SaTML). The final version will be available on IEEE Xplore.}
\newcommand\copyrightnotice{%
\begin{tikzpicture}[remember picture,overlay]
\node[anchor=south,yshift=25pt] at (current page.south) {\fbox{\parbox{\dimexpr0.62\textwidth-\fboxsep-\fboxrule\relax}{\copyrighttext}}};
\end{tikzpicture}%
}

\begin{document}

\title{Streaming Private Continual Counting via Binning}

\author{Joel Daniel Andersson\\
\textit{Basic Algorithms Research Copenhagen}\\
\IEEEauthorblockA{\textit{University of Copenhagen}\\
Copenhagen, Denmark \\
\texttt{jda@di.ku.dk}}
\and
\IEEEauthorblockN{Rasmus Pagh}
\textit{Basic Algorithms Research Copenhagen}\\
\IEEEauthorblockA{\textit{University of Copenhagen}\\
Copenhagen, Denmark \\
\texttt{pagh@di.ku.dk}}
}

\maketitle

\begin{abstract}
In differential privacy, \emph{continual observation} refers to problems in which we wish to continuously release a function of a dataset that is revealed one element at a time.
The challenge is to maintain a good approximation while keeping the combined output over all time steps differentially private.
In the special case of \emph{continual counting} we seek to approximate a sum of binary input elements.
This problem has received considerable attention lately, in part due to its relevance in implementations of differentially private stochastic gradient descent.
\emph{Factorization mechanisms} are the leading approach to continual counting, but the best such mechanisms do not work well in \emph{streaming} settings since they require space proportional to the size of the input.
In this paper, we present a simple approach to approximating factorization mechanisms in low space via \emph{binning}, where adjacent matrix entries with similar values are changed to be identical in such a way that a matrix-vector product can be maintained in sublinear space.
Our approach has provable sublinear space guarantees for a class of lower triangular matrices whose entries are monotonically decreasing away from the diagonal.
We show empirically that even with very low space usage we are able to closely match, and sometimes surpass, the performance of asymptotically optimal factorization mechanisms.
Recently, and independently of our work, Dvijotham et al.~have also suggested an approach to implementing factorization mechanisms in a streaming setting.
Their work differs from ours in several respects: It only addresses factorization into \emph{Toeplitz} matrices, only considers \emph{maximum} error, and uses a different technique based on rational function approximation that seems less versatile than our binning approach.
\end{abstract}


\section{Introduction}
\copyrightnotice

Analyzing private data involves creating methods to query datasets while protecting the privacy of the individuals within them.
A leading framework in this field is \emph{differential privacy}~\cite{dp_2006}, which provides precise probabilistic assurances regarding individual privacy.
The balance between the usefulness of the data and the privacy protections offered by differentially private algorithms is significantly influenced by the nature of the queries being made.

\emph{Continual observation} \cite{dwork_differential_2010, dwork_algorithmic_2013} constitutes a particularly interesting scenario.
In this setting, data entries are received sequentially and each new entry requires an immediate query response based on the entire dataset accumulated up to that point.
The challenge comes from the fact that each data entry can affect all future outputs.
A fundamental problem in this context is \emph{counting under continual observation}~\cite{dwork_differential_2010, chan_differentially_2012}, \emph{continual counting} for short.
The input stream is $\vec{x}_1, \vec{x}_2, \dots \vec{x}_n \in \{0,1\}$, arriving one item at a time, where $\vec{x}_t$ is received at time $t$, and the task is to output the number of $1$s observed so far.
Two input streams $\vec{x}$ and $\vec{x}'$ are considered neighboring if they differ at exactly one index.
Although this is one of the most basic problems in the area, there are still gaps between the lower and upper limits for both approximate and pure differential privacy~\cite{dwork_differential_2010, henzinger_almost_2023, fichtenberger_constant_2023, dwork_rectangle_queries_2015, cohen2024lower, dvijotham2024efficient}.

In recent years, the continual counting problem has gained renewed attention due to its use as a subroutine for private learning \cite{kairouz_practical_2021,denissov_improved_2022,choquette-choo_multi-epoch_2022,privacyamplification2023}.
For some private learning settings, the useful quantity to make private turns out to be \emph{prefix sums on sequences of gradient vectors} \cite{thakurta2013,kairouz_practical_2021}, a problem that can be reduced to continual counting.
In deep learning settings in particular, the regime of parameters is such that only space-efficient mechanisms for continual counting are viable, since there is a blow-up in the space proportional to the dimension of the gradients (i.e., the numbers of parameters in the model).
Dvijotham, McMahan, Pillutla, Steinke, and Thakurta~\cite{dvijotham2024efficient} provide an example of a learning task where, given 32-bit precision, over 40 petabytes of memory would be needed to run the mechanism with space complexity $\Omega(n)$.

For approximate differential privacy, the setting in which deep private learning typically takes place, all known mechanisms, excluding special cases such as sparse inputs, have an expected error of $\Theta(\log n)$ per step, and no improvement in the asymptotic dependence on $n$ has been achieved since the first non-trivial algorithms in~\cite{dwork_differential_2010,chan_private_2011}.
The best known continual counting mechanisms belong to the class of \emph{factorization mechanisms} \cite{li_matrix_2015}.
Let $A$ be the matrix such that $A\vec{x}$ is the prefix sum vector for $\vec{x}$. A factorization mechanism is characterized by a factorization of $A$ into a product of matrices, $A=LR$.
Recent work \cite{henzinger_almost_2023} showed that the factorization where $L=R=\sqrt{A}$ (a matrix also studied in \cite{bennett77}) achieves optimal mean squared error, up to lower order vanishing terms, across all factorizations, but this factorization mechanism has a space complexity of $\Omega(n)$.
Thus, alternative factorization mechanisms running in sublinear space have been proposed \cite{denissov_improved_2022,andersson2024smooth,dvijotham2024efficient}.

\subsection{Our contributions}

We propose a framework for factorization mechanisms using a class of lower-triangular matrices $L\in\R^{n\times n}$ that yield low space complexity, as defined next.
\begin{definition}[Space complexity of $L$ for streaming]
    Let $L\in\R^{n\times n}$ be a lower-triangular matrix and let $\vec{z}\in\R^{n}$ arrive as a stream where $\vec{z}_i$ arrives at time $i$ and we immediately need to output $(L\vec{z})_i$.
    A streaming algorithm is defined by a state update function that maps the step number, current state, and input $\vec{z}_i$ to a new state and an output $(L\vec{z})_i$.
    We consider \emph{linear} streaming algorithms where the state is a vector in $\R^q$ and the update function for each step number $i$ is a linear function of the state and the input $\vec{z}_i$.
    The \emph{space complexity} of~$L$ is the smallest $q$ for which there exists a (linear) streaming algorithm outputting~$L\vec{z}$.
\end{definition}
Note that the state update function may depend on $L$ and the output is not recorded in the state, so neither $L$ nor the output $L\vec{z}$ are counted in the space usage.
The class of matrices we consider define, for each row, a partition of the set of nonzero entries into intervals such that the entries in each interval of the row are identical.
Furthermore the intervals align between consecutive rows such that each interval of row $i$ is either the union of intervals in row $i-1$ or the singleton interval $\{i\}$.
This allows us to compute $(L\vec{z})_i$ in small space by keeping track of the sum of entries in $\vec{z}$ corresponding to each interval.
We call this class of matrices \emph{binned} matrices, and the structure of the binning decides the space complexity.

Our main contribution is an algorithm (\Cref{alg:matrix_to_binning}) that approximates a given matrix $L$ by a binned matrix $\hat{L}$.
The concept is demonstrated in \Cref{fig:binning_examplified}.
\begin{figure}[t!]
\centering
\subfigure[Lower triangular matrix $L$.]{\includegraphics{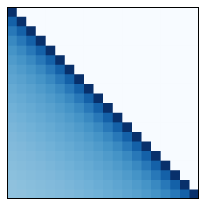}\label{fig:unbinned_example}}
\subfigure[$L$ after binning.]{\includegraphics{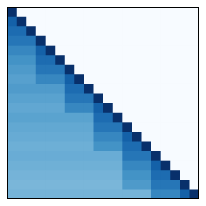}\label{fig:binned_example}}
\caption{
    Illustration of the concept of binning on a 20-by-20 matrix.
    Each square corresponds to an entry in the matrix, with a darker shade illustrating a larger value. 
    The binned matrix in \Cref{fig:binned_example} uses at most 4 intervals per row, as illustrated by each row of the matrix having at most 4 shades.
}\label{fig:binning_examplified}
\end{figure}
When applied to an invertible matrix $L$ of a factorization $LR$, we get a new factorization $\hat{L}\hat{R}$ where $\hat{R} = \hat{L}^{-1}LR$.
We show that if $L$ comes from a particular class of matrices, namely those satisfying \Cref{def:nice_matrix}, the binned factorization offers good utility at low space complexity.
In the particular case of continual counting and the (almost optimal) factorization $L=R=B=\sqrt{A}$ \cite{henzinger_almost_2023,fichtenberger_constant_2023,bennett77}, we show the following result.
\begin{theorem}\label{thm:main_bennett}
    Let $B\in\R^{n\times n}$ be the Bennett matrix.
    Then for constant $\xi \in (0,1)$ there is a factorization $\hat{L}\hat{R}=B^2$ where:
    \begin{itemize}
        \item $\hat{L}\hat{R}$ satisfies $\mse(\hat{L},\hat{R}) \leq (1+\xi)\mse(B,B)$ and $\maxerr(\hat{L},\hat{R}) \leq (1+\xi)\maxerr(B,B)$.
        \item $\hat{L}$ has space complexity $O_{\xi}(\sqrt{n}\log n)$.
        \item Given $\vec{z}\in\R^n$ and random access to $B$, $\hat{L}\vec{z}$ can be computed incrementally in time $O_\xi(\sqrt{n}\log n)$ per output.
    \end{itemize}
\end{theorem}

\noindent
That is, we achieve the same \emph{mean (expected) squared error} and \emph{maximum (expected) squared error} up to a factor of $(1+\xi)$ in sublinear space and time per output.
We supplement our theoretical result with empirical results up to $n\approx 10^4$ that confirm the space- and error efficiency of our method, even at space usage much lower than what our theoretical results require.
Surprisingly, we do not only nearly match the error of the asymptotically optimal factorization, but often \emph{improve} the error.
We also provide theoretical and empirical results for the more general square-root factorization of the counting matrix with momentum and weight decay \cite{KalLamp24}, where our approach is equally suitable.

\subsection{Sketch of technical ideas}
We start out by considering an arbitrary factorization $LR$ and investigate how the norms $\norm{L}_F$, $\norm{L}_{2\to\infty}$ and $\norm{R}_{1\to 2}$ change under a perturbation of $L$ (adjusting $R$ such that the matrix product is preserved).
For this we need $L\in\R^{n\times n}$ to be invertible, and consider perturbations $P\in\R^{n\times n}$ where for $\eta, \mu \geq 0$, we have that $\abs{P_{i,j}} \leq\eta \abs{L_{i,j}} + \mu$.
If $P$ is chosen so that $\hat{L} = L + P$ is invertible, we can then define $\hat{R} = \hat{L}^{-1}LR$, and get a new factorization $\hat{L}\hat{R} = LR$.
We show that the error of $\hat{L}\hat{R}$ exceeds that of $LR$ by no more than a multiplicative factor $(1+\xi)$, where $\xi$ depends on $\mu, \eta$ and $L$, and can be made arbitrarily close to zero.
We then proceed to describe a class of lower-triangular matrices, referred to as \emph{binned} matrices where row $i$ is assigned a partition $\B^i$ of $[i]=\{1,2, \dots, i\}$ and where the matrix is constant on each interval of the partition.
Furthermore, we require that $\B^{i+1}$ can be constructed by merging intervals in $\B^{i}$ and then adding $\{i+1\}$ to it.
Such matrices have space complexity equal to the maximum number of intervals in any of the partitions $\B^1, \dots, \B^n$, which can be $o(n)$.
We then give an algorithm for approximating a matrix $L$ by a binned matrix $\hat{L}$, where the central idea is to group elements based on them being within a multiplicative constant $c$ of each other, up to a small additive constant.
For a class of matrices, which in particular includes the square-root factorization of the counting matrix, we show that the resulting perturbation $P = \hat{L} - L$ can be bounded entry-wise, implying good error, with a binning of sublinear size, resulting in small space usage.
(Empirically, logarithmic space seems to suffice.)

\subsection{Limitations}
Our technique is only known to apply to a class of factorization mechanisms in which the left matrix (defining the noise distribution) satisfies a certain monotonicity condition (Definition~\ref{def:nice_matrix}).
Without conditions on this matrix, our binning might lead to an increase in the sensitivity of the factorization that would significantly increase the amount of noise added by the mechanism.
The experimental results for counting use the exact sensitivity of the factorization rather than the upper bound from our theoretical analysis --- using this sharper bound requires a quadratic time preprocessing algorithm to be run.
We stress that this computation only depends on $n$ and can therefore be done ``offline'' before any data is received.
If the exact sensitivity is stored, it can be reused for any stream with length at most $n$.

Our bound on space usage does \emph{not} include the space for storing the matrix $L$, and we furthermore assume that an entry of $L$ can be accessed in constant time.
If the elements in $L$ can be computed efficiently on the fly (like in our key examples), then there is no need for storing $L$.
If $L$ has no such structure, we note that for many applications, the attention paid to space usage is derived from use cases where we are interested in running the mechanism on $d$-dimensional vectors where $d$ may be larger than the number of outputs, e.g., prefixes of gradients.
In such applications, the space needed for storing $L$ explicitly may be dominated by the space needed for storing the set of $d$-dimensional noise vectors.
In particular, for deep learning the regime of parameters is often such that $d \gg n$.

\subsection{Related work}
Central to our work is the \emph{factorization mechanism} \cite{li_matrix_2015}, also known as \emph{matrix mechanism}, framework to which all mechanisms discussed here belong.
It addresses the problem of releasing linear queries $A\in\R^{q\times n}$ when applied to a dataset $\vec{x}\in\R^n$, and solves it by factoring $A$ into matrices $L, R$ where $A=LR$, privately releasing $R\vec{x} + \vec{z}$ where the noise $\vec{z}$ is scaled to the sensitivity of $R$, and then outputting $L(R\vec{x} + \vec{z})$.

While our techniques derive from factorization mechanisms, and our results extend to more general query matrices, the problem that motivates us is \emph{continual counting}.
The first mechanism to achieve polylogarithmic error for this problem was the \emph{binary mechanism}~\cite{dwork_differential_2010,chan_private_2011}, and it did so in space complexity $O(\log n)$.
Since then no asymptotic improvements in the error have been achieved without further assumptions on the input (e.g., sparsity \cite{dwork_rectangle_queries_2015}).

Recent work identified that the Toeplitz matrix factorization $L=R=\sqrt{A}$ for continual counting achieved optimal, up to lower-order vanishing terms, mean squared error over all factorizations~\cite{henzinger_almost_2023} and state-of-the-art maximum squared error~\cite{fichtenberger_constant_2023}.
However, the structure of $L$ gives the mechanism a space complexity of $\Omega(n)$, rendering it impractical for large-scale applications.
The same is true of other factorizations corresponding to aggregation matrices used in data analysis and machine learning applications~\cite{henzinger2024unifying,KalLamp24}.

More recently, and independently of our work, Dvijotham, McMahan, Pillutla, Steinke, and Thakurta \cite{dvijotham2024efficient}, identified the same problem and designed a near-optimal solution with space complexity $\Theta(\log^2 n)$ and the same maximum squared error up to a $(1+o(1))$ multiplicative factor.
They prove near-optimality with respect to the maximum squared error but do not state any results for mean squared error.
Their solution is based on approximating $L$ and $R$ by Toeplitz matrices $\hat{L}$ and $\hat{R}$ where the values on subdiagonals of $\hat{L}$ satisfy a recursion that facilitates computing $L\vec{z}$ in low space.
Their technique could potentially be extended to factorizations of different $A$, but their existing analysis only holds for the lower-triangular all-1s matrix and relies on all matrices involved being Toeplitz.
Our approach does not need any assumption of Toeplitz structure and is arguably simpler.
It has a weaker theoretical space bound, but our empirical work shows that it works well with much less space, suggesting that our analysis is not tight.

Our main algorithm for producing binned matrices is similar in structure to \emph{weight-based merging histograms} of Cohen and Strauss~\cite[Section 5]{CohenS03}.
They considered (approximately) outputting decaying prefix sums of a stream $x_1,\dots, x_T$, i.e., releasing $\sum_{i=1}^t g(t-i)x_i$ where $g$ is a non-increasing function, while using little space for storing the stream.
This can be translated into our setting as approximating a Toeplitz left factor matrices $L$ by a matrix $\hat{L}$ with low space complexity.
Though our binned matrices are constructed in a way similar to weight-based merging histograms, we stress that our objectives are very different and require an entirely different and more sophisticated analysis.
We use the term \enquote{binning} that we feel better expresses what happens to the rows of a matrix.

\section{Introducing the problem}
\subsection{Preliminaries}
We let $\N = \{1, 2, \dots\}$ denote the set of all positive integers, use the notation $[a, b]$ for $a, b\in\N$, $a \leq b$ to denote the set of integers $\{a, a+1, \dots, b\}$, and in particular let $[a] = [1, a]$.

\medskip

\subsubsection{Linear algebra}
For a matrix $A\in\R^{m\times n}$, let $A_{i,j}\in\R$ for $i\in [m], j\in[n]$ denote an entry.
In the case where $[a, b]\subseteq [n]$ and $A_{i, j}$ assumes the same value for all $j\in[a, b]$, we will abuse notation and let $A_{i, [a, b]}$ represent this scalar.
We reserve the notation $\abs{A}$ for the matrix with entries $\abs{A}_{i, j} = \abs{A_{i, j}}$.
$\norm{A}_2$ denotes the operator norm of a matrix and if $A$ is invertible then $\kappa(A) = \norm{A}_2\norm{A^{-1}}_2$ defines the condition number of $A$.
Whenever a matrix product $AB$ is given without specifying the dimensions of $A\in\R^{m\times n}$ and $B\in\R^{m'\times n '}$, we are implicitly imposing the restriction $n=m'$.%

\medskip

\subsubsection{Computational model}
Our computational model assumes that any real number can be stored in constant space, and that the multiplication or addition of two reals can be performed in constant time.
When a data structure is said to be \emph{random access}, we mean that an arbitrary entry from it can be retrieved in constant time.

\medskip

\subsubsection{Differential privacy}

For a mechanism to be considered differentially private, we require that the output distributions of any two neighboring inputs, $D\sim D'$, are indistinguishable.
\begin{definition}[$(\epsilon, \delta)$-Differential Privacy~\cite{dwork_algorithmic_2013}]
    A randomized algorithm $\M : \mathcal{X}^n\to\mathcal{Y}$ is ($\epsilon, \delta)$-differentially private if for all $S\subseteq\mathsf{Range}(\M)$ and all neighboring inputs $\mathcal{D}, \mathcal{D}'\in \mathcal{X}^n$, the neighboring relation written as $\mathcal{D}\sim \mathcal{D}'$, we have that:
    \begin{equation*}
        \Pr[\M(\mathcal{D})\in S] \leq \exp(\epsilon)\Pr[\M(\mathcal{D}')\in S] + \delta\, .
    \end{equation*}
    We refer to $(\epsilon, 0)$-differential privacy as $\epsilon$-differential privacy, or \emph{pure} differential privacy.
\end{definition}
When releasing a real-valued function $f(\D)$ taking values in $\R^d$, we can achieve \approxdp{} by adding Gaussian noise scaled to the $\ell_2$-sensitivity of~$f$.
\begin{lemma}[Gaussian Mechanism \cite{dwork_algorithmic_2013}]
    Let $\epsilon\in(0, 1), \delta\in(0,1)$, and $f : \mathcal{X}^n \to \R^d$ be a function with $\ell_2$-sensitivity $\Delta \coloneqq \max_{\D\sim\D'}\norm{f(\D)-f(\D')}_2$.
    For a given data set $\D\in\mathcal{X}^n$, the mechanism that releases $f(\D) + \mathcal{N}(0, C_{\epsilon, \delta}\Delta^2)^d$ satisfies \approxdp{} where $C_{\epsilon, \delta}=2\ln(1.25/\delta)/\epsilon^2$.
\end{lemma}

\subsection{Factorization mechanisms.}
Most known mechanisms for continual counting (at least all such that produce unbiased outputs) can be expressed in the \emph{factorization mechanism} framework~\cite{li_matrix_2015}, also known as the \emph{matrix mechanism} framework.
The general formulation describes the setting where the input is $\vec{x}\in\mathbb{R}^n$, for some neighbouring relation \enquote{$\sim$}, and we want to privately release linear queries $f(\vec{x}) = A\vec{x}\in\mathbb{R}^q$ on $\vec{x}$, where $A\in\R^{q\times n}$.

Given a factorization $L\in\R^{q\times m},\, R\in\R^{m\times n}$ such that $A=LR$, the framework states that the mechanism $\M(\vec{x}) = A\vec{x} + L\vec{z}$ where $\vec{z}$ is drawn from $\gauss(0, C_{\delta,\epsilon}\Delta^2)^m$, $\Delta = \max_{\vec{x}\sim \vec{x}'}\norm{ R(\vec{x}-\vec{x}') }_2$, and $C_{\epsilon, \delta}$ is a constant depending on $\epsilon, \delta$, satisfies $(\epsilon, \delta)$-DP.
The intuition behind the mechanism is that we encode our input $\vec{x}$ as $R\vec{x}$, release it privately as $R\vec{x} + \vec{z}$ using the Gaussian mechanism, and then output $\M(\vec{x}) = L(R\vec{x} + \vec{z}) = A\vec{x} + L\vec{z}$ as post-processing.
Notably, each coordinate of the output $\M(\vec{x})$ is in expectation equal to $A\vec{x}$, i.e., the mechanism is unbiased.

Applied to our setting, let $A\in\R^{n\times n}$ be the lower-triangular all-1s matrix and let the input vectors $\vec{x}, \vec{x}'\in\{0, 1\}^n$ be considered neighboring if $\vec{x}$ is equal to $\vec{x}'$ up to a bit flip at one unique coordinate.
Any factorization $L, R$ where $A=LR$ thus constitutes a mechanism for continual counting with $\Delta = \norm{R}_{1\to 2}$, where $\norm{\cdot}_{1\to 2}$ denotes the maximum $\ell_2$-norm of all columns in a matrix.

There are two common error metrics for continual counting: \emph{mean squared error} (MSE) and \emph{maximum squared error} (MaxSE).
The mean squared error, is given as
\begin{align*}
    \mse(\M)
    &= \frac{1}{n}\expectation_{\M}[\norm{\M(\vec{x}) - A\vec{x}}_2^2]\\
    &= \frac{1}{n}\expectation_{\vec{z}}[\norm{L\vec{z}}_2^2]
    = \frac{\dpconst}{n}\norm{L}_F^2\norm{R}_{1\to 2}^2
\end{align*}
where $\norm{\cdot}_F$ is the Frobenius norm.
Analogously for the maximum (expected) squared error, we get
\begin{align*}
    \maxerr(\M)
    &= \max_i \expectation_{\M}[(\M(\vec{x}) - A\vec{x})_i^2]\\
    &= \max_i \expectation_{\vec{z}}[(L\vec{z})_i^2]
    = \dpconst\norm{L}_{2\to\infty}^2\norm{R}_{1\to 2}^2
\end{align*}
where $\norm{\cdot}_{2\to\infty}$ denotes the maximum $\ell_2$-norm of all \emph{rows} in the matrix.
Intuitively, since $\M$ is unbiased, these two error metrics correspond to the \emph{average variance} and \emph{maximum variance} for $n$ outputs.
For notational convenience, we will superload this notation and use $\mse(L, R)$ and $\maxerr(L, R)$ for the error of a matrix mechanism built on a factorization $L, R$.

\subsection{(Almost) optimal factorizations}
For $(\epsilon, \delta)$-DP, it is known \cite{henzinger_almost_2023} that choosing $L=R=\sqrt{A}$ gives a continual counting mechanism with optimal mean squared error up to lower-order vanishing terms across all factorizations.
The same factorization also improved the best known \emph{maximum expected squared error} \cite{fichtenberger_constant_2023}.
For this error, it was recently shown to be optimal across all \emph{lower-triangular Toeplitz} factorizations \cite{dvijotham2024efficient}.
We call the corresponding matrix $B=\sqrt{A}$ the \emph{Bennett matrix}~\cite{bennett77}.
Given these strong error bounds, choosing the factorization $L=R=B$ would seem to be the obvious choice in practical settings.

\subsection{Space matters}
As identified in \cite{denissov_improved_2022} and elsewhere,
the space efficiency of the counting mechanisms has a large impact for large-scale applications in private learning.
For a given stream $\vec{x}_1, \vec{x}_2, \dots, \vec{x}_n$ we need to add $(L\vec{z})_t$ at time $t$ to the output to achieve privacy.
Importantly, the memory needed to compute $\M(\vec{x})_t$ will therefore directly depend on the structure of $L$.
If $L = I\in\R^{n\times n}$, the identity matrix, then the memory needed is constant, as each entry $\vec{z}_t$ is used once at time $t$ and then never again.
If $L$ and $R$ correspond to the factorization of the \emph{binary mechanism} \cite{chan_private_2011,dwork_differential_2010}, then the needed memory will be $O(\log n)$.
Unfortunately, in the particular case of $L=R=B$, the structure of $L$ seems to require $\Omega(n)$ space.

This begs the question: can we find a factorization where $L\approx B$, $R\approx B$ and which runs in sublinear space?
In particular, can we approximate its mean and maximum squared error?

\section{Approximation by binning}

Given a factorization $LR = A$ where $L$ is lower-triangular and invertible, we are looking to approximate it, producing a new factorization $\hat{L}\hat{R} = A$ where $\mse(\hat{L}, \hat{R})\approx\mse(L, R)$, $\maxerr(\hat{L}, \hat{R})\approx\maxerr(L, R)$, and $\hat{L}$ has space complexity $o(n)$.
We will approach the problem by expressing $\hat{L}$ as the sum of $L$, and a perturbation to each entry~$P$.
The big picture idea is that, for some class of $L$, perturbing its elements a little might be sufficient to achieve sublinear space.
Furthermore, if the perturbation $P$ is small enough and preserves the invertibility of $L$, then the requisite norms of $\hat{L}$ and $\hat{R}=\hat{L}^{-1}A$ will be close to those of $L$ and~$R$.
We address this problem next.

\subsection{Small perturbations preserve our norms of interest}

We start by formalizing the intuition that sufficiently small perturbations preserve norms.

\begin{definition}\label{def:perturbed_matrices}
    For matrices $L, P\in\R^{n\times n}$, we call $L+P$ an $(\eta, \mu)$-perturbation of $L$ if given $\eta, \mu \geq 0$, we have that $\abs{P_{i, j}} \leq \eta \abs{L_{i, j}} + \mu$ for all $i,j\in[n]$.
\end{definition}

\begin{lemma}\label{lemma:good_frobenius}
  Let $L, P\in\R^{n\times n}$ and $L+P$ be an $(\eta, \mu)$-perturbation of $L$.
  Then $\norm{L+P}_{F} \leq (1+\eta) \norm{L}_{F} + \mu n$ and $\norm{L+P}_{2\to\infty} \leq (1+\eta) \norm{L}_{2\to\infty} + \mu\sqrt{n}$.
\end{lemma}
\begin{proof}
    We have that
    \begin{alignat*}{1}
    \norm{L + P}_F & \leq \norm{L}_F + \norm{P}_F\\
    & \leq \norm{L}_F + \norm{\eta\abs{L} + \mu J}_F\\
    & \leq (1+\eta)\norm{L}_F + \mu\norm{J}_F\\
    & = (1+\eta)\norm{L}_F + \mu n,
    \end{alignat*}
    where $J\in\R^{n\times n}$ is the all-1s matrix, proving the statement.
    The first and last inequalities are cases of the triangle inequality, whereas the second step uses the definition of a $(\eta, \mu)$-perturbation and the fact that Frobenius norm only depends on the absolute value of each entry.
    The statement for $\norm{L+P}_{2\to\infty}$ follows from an identical argument, except that $\norm{J}_{2\to\infty} = \sqrt{n}$.
\end{proof}
Unsurprisingly, a small perturbation of $L$ does not affect the norms we are interested in by much.
To show that $\hat{L} = L + P$ induces a factorization $\hat{L}\hat{R} = LR$ with  error comparable to $LR$ requires more work.
We start by bounding $\norm{\hat{R}}_{1\to 2}$.
\begin{lemma}\label{lemma:good_sensitivity}
  Let $L,R,P$ be matrices where $L$ and $L+P$ are invertible, and $\norm{L^{-1}}_2\cdot\norm{P}_2 \leq 1/2$.
  Suppose that $\hat{R} = (L+P)^{-1} LR$.
  Then $\norm{\hat{R}}_{1\to 2} \leq \big(1 + 2\norm{P}_2\cdot\norm{L^{-1}}_2\big)\norm{R}_{1\to 2}$.
\end{lemma}
\begin{proof}
    We begin by upper and lower bounding $\norm{L(\hat{R} - R)}_{1\to 2}$.
    The following fact will be used for both bounds.
    \begin{fact}\label{fact:AB-norm}
        For any matrix product $AB$, we have that $\norm{AB}_{1\to 2} \leq \norm{A}_2\norm{B}_{1\to 2}$.
        If furthermore $A$ is invertible, then $\norm{AB}_{1\to 2} \geq \norm{B}_{1\to 2} / \norm{A^{-1}}_2$.
    \end{fact}
        To see why Fact~\ref{fact:AB-norm} holds we begin with the lower bound for invertible $A$.
        Note that the standard inequality $\norm{A\vec{x}}_2 \leq \norm{A}_2\norm{\vec{x}}_2$ also implies $\norm{A\vec{x}}_2 \geq \norm{\vec{x}}_2 / \norm{A^{-1}}_2$ via
        \begin{equation*}
           \norm{\vec{x}}_2 = \norm{A^{-1}A\vec{x}}_2 \leq \norm{A^{-1}}_2 \norm{A\vec{x}}_2\enspace.
        \end{equation*}
        We can thus derive a lower bound as follows
        \begin{align*}
            \norm{AB}_{1\to 2} = \sup_{\norm{\vec{x}}_1 =1} \norm{AB\vec{x}}_2 \geq \sup_{\norm{\vec{x}}_1 =1} \frac{\norm{B\vec{x}}_2}{\norm{A^{-1}}_2} = \frac{\norm{B}_{1\to 2}}{\norm{A^{-1}}_2}\,,
        \end{align*}
        The upper bound is shown in an analogous way, concluding the proof of Fact~\ref{fact:AB-norm}.

    For the upper bound of Lemma~\ref{lemma:good_sensitivity}, note that 
    \[L(\hat{R} - R) = L\hat{R} - A = (L-\hat{L})\hat{R} = -P\hat{R}, \]
    and thus $\norm{L(\hat{R}-R)}_{1\to 2} \leq \norm{P}_2\norm{\hat{R}}_{1\to 2}$.
    For the lower bound, we directly derive
    \[\norm{L(\hat{R}-R)}_{1\to 2} \geq \norm{\hat{R}-R}_{1\to 2} / \norm{L^{-1}}_2 \enspace .\]
    Combining the two, we get $\norm{\hat{R} - R}_{1\to 2} \leq \norm{P}_2\cdot\norm{L^{-1}}_2\cdot \norm{\hat{R}}_{1\to 2}$.
    Applying the triangle inequality we can further get $\norm{\hat{R}-R}_{1\to 2} \geq \norm{\hat{R}}_{1\to 2} - \norm{R}_{1\to 2}$, which combined with the preceding inequality and some re-arranging yields
    \[(1-\norm{P}_2\cdot\norm{L^{-1}}_2)\norm{\hat{R}}_{1\to 2} \leq \norm{R}_{1\to 2} \enspace . \]
    By our assumption $(1-\norm{P}_2\cdot\norm{L^{-1}}_2) \geq 1/2$, and so we can divide by it, resulting in
    \begin{alignat*}{1}
    \norm{\hat{R}}_{1\to 2} & \leq \norm{R}_{1\to 2} / (1 - \norm{P}_2\cdot\norm{L^{-1}}_2) \\
    &\leq \big(1 + 2\norm{P}_2\cdot\norm{L^{-1}}_2\big)\norm{R}_{1\to 2},
    \end{alignat*}
    where the last step uses that $\frac{1}{1-x} \leq 1 + 2x$ for $0 \leq x \leq 1/2$.
    This concludes the argument.
\end{proof}
Intuitively, if the perturbation $P$ is \enquote{sufficiently small} compared to $L$, then the sensitivity of our new factorization $\hat{L}\hat{R}$ is bounded by that of $LR$.
To this end, our next lemma shows that, if we restrict the left factor matrix $L$ to be non-negative, then we can bound $\norm{P}_2$ using $\norm{L}_2$.
\begin{lemma}\label{lemma:bound_operator_norm_of_perturbation}
  Let $L\in\R^{n\times n}_{\geq 0}, P\in\R^{n\times n}$ and $L+P$ be an $(\eta, \mu)$-perturbation of $L$. %
  Then $\norm{P}_{2} \leq \eta \norm{L}_{2} + \mu n$. %
\end{lemma}
\begin{proof}
    Let $U = \eta L + \mu J$ where $J\in\R^{n\times n}$ is the all-1s matrix.
    We have that $\norm{P}_2 \leq \norm{\abs{P}}_2$, since for any $\vec{v}\in\R^n$, $\abs{P}$ applied to $\abs{v}$ will always yield an equal or greater norm.
    Similarly, we also have that $\norm{\abs{P}}_2 \leq \norm{U}_2$, since (1) any unit vector $\vec{v}$ maximizing $\norm{\abs{P}\vec{v}}_2$, will have the same sign in each entry, and (2) $\abs{P_{i, j}} \leq U_{i,j}$ by definition of the $(\eta, \mu)$-perturbation.
    Taken together this gives us
    \[\norm{P}_2 \leq \norm{U}_2 \leq \eta\norm{L}_2 + \mu\norm{J}_2 = \eta\norm{L}_2 + \mu n \enspace . \]
\end{proof}

We are now ready to state our main lemma for perturbed factorizations.
\begin{lemma}\label{lemma:main_approximation}
  Let $LR$ be a factorization, and let $\hat{L}$ be an $(\eta, \mu)$-perturbation of $L\in\R^{n\times n}_{\geq 0}$ for $\eta = \frac{\xi}{144\kappa(L)}$ and $\mu = \frac{\xi\norm{L}_2}{144n\kappa(L)}$ where $0 \leq \xi \leq 24$.
  Also let $L$ and $\hat{L}$ be invertible, and define $\hat{R} = \hat{L}^{-1} LR$.
  Then $\mse(\hat{L},\hat{R})\leq (1+\xi)\mse(L, R)$ and $\maxerr(\hat{L},\hat{R})\leq (1+\xi)\maxerr(L, R)$.
\end{lemma}
\begin{proof}
    We start by proving the statement for the mean squared error.
    Assume $\eta \leq \frac{1}{6\kappa(L)}$ and $\mu \leq \frac{\norm{L}_2}{6n\kappa(L)}$.
    First, note that \Cref{lemma:good_frobenius} gives a bound on $\norm{\hat{L}}_F$.
    Secondly, invoking \Cref{lemma:bound_operator_norm_of_perturbation} we have that
    $\norm{P}_2
    \leq \eta\norm{L}_2 + \mu n
    \leq \frac{1}{3\norm{L^{-1}}_2}$ by \Cref{def:perturbed_matrices} and our assumptions on $\eta, \mu$.
    This implies $\norm{P}_2\norm{L^{-1}}_2 \leq 1/2$ and allows us to use \Cref{lemma:good_sensitivity} for a bound on $\norm{\hat{R}}_{1\to 2}$.
    For the mean squared error, we thus have
    \begin{align*} 
        &\frac{\mse(\hat{L}, \hat{R})}{\mse(L, R)}
        = \frac{\norm{\hat{L}}_F^2\norm{\hat{R}}_{1\to 2}^2}{\norm{L}_F^2\norm{R}_{1\to 2}^2}\\
        &\leq \bigg(1+\eta+\frac{\mu n}{\norm{L}_F}\bigg)^2 \Big(1 + 2\norm{P}_2\norm{L^{-1}}_2\Big)^2\\
        &\leq \bigg(1+\eta+\frac{\mu n}{\norm{L}_F}\bigg)^2 \big(1 + 2(\eta\norm{L}_2 + \mu n)\norm{L^{-1}}_2)^2\\
        &= (1+\eta)^2 \bigg(1+\frac{\mu n}{(1+\eta)\norm{L}_F}\bigg)^2\\
        &\qquad \times \big(1 + 2\eta\kappa(L)\big)^2 \bigg(1 + \frac{2\mu n \norm{L^{-1}}_2}{1 + 2\eta\kappa(L)}\bigg)^2\\
        &\leq (1+\eta)^2 (1 + 2\eta\kappa(L))^2 \big(1 + 2\mu n \norm{L^{-1}}_2\big)^2\Big(1+\frac{\mu n}{\norm{L}_F}\Big)^2\\
        &\leq 1 + 24\bigg(\eta(1 + 2\kappa(L)) + \frac{\mu n(1 + 2\norm{L}_F\norm{L^{-1}}_2)}{\norm{L}_F} \bigg)\\
        &\leq 1 + 72\Big(\eta\kappa(L) + \mu n \norm{L^{-1}}_2 \Big) = 1 + \xi \leq 25
    \end{align*}
    The first inequality is using \Cref{lemma:good_frobenius} and \Cref{lemma:good_sensitivity}, the second inequality \Cref{lemma:bound_operator_norm_of_perturbation}.
    The fourth inequality is using the fact that each factor is of the form $(1+\omega)^2$ where $\omega \leq 1/3$, in which case $(1+\omega)^2 \leq 1 + 3\omega$, and in general that $\Pi_{i\in[n]}(1+\omega_i)^2 \leq \Pi_{i\in[n]} (1+3\omega_i) \leq 1 + 3\cdot 2^{n-1}\sum_{i\in[n]} \omega_i$ if all $\omega_i \leq 1/3$.
    Each value $\omega_i$ can be confirmed to be less than or equal to $1/3$ by inspection based on our assumptions on $\eta, \mu$.
    For the two non-immediate cases, note that $1=\norm{I}_2 = \norm{L^{-1}L}_2 \leq \norm{L^{-1}}_2 \norm{L}_2 = \kappa(L) \leq \norm{L^{-1}}_2 \norm{L}_F$ and thus our assumptions imply $\eta \leq 1/6$ and $\mu \leq \frac{\norm{L}_F}{6n}$.
    The same matrix norm inequality is used in the second to last inequality.
    It follows that choosing $\eta = \frac{\xi}{144\kappa(L)}, \mu = \frac{\xi}{144n\norm{L^{-1}}_2} = \frac{\xi\norm{L}_2}{144n\kappa(L)}$ for any $\xi \leq 24$ gives a $(1+\xi)$-approximation of the error.

    For the maximum squared error, the same argument can be used with the exception that we are bounding $\norm{L}_{2\to\infty}$ at the first step.
    \begin{align*}
        & \maxerr(\hat{L}, \hat{R}) / \maxerr(L, R)\\
        & = \frac{\norm{\hat{L}}_{2\to\infty}^2\norm{\hat{R}}_{1\to 2}^2}{\norm{L}_{2\to\infty}^2\norm{R}_{1\to 2}^2}\\
        &\leq \bigg(1+\eta+\frac{\mu \sqrt{n}}{\norm{L}_{2\to \infty}}\bigg)^2 \Big(1 + 2\norm{P}_2\norm{L^{-1}}_2\Big)^2\\
        &\leq \bigg(1+\eta+\frac{\mu n}{\norm{L}_F}\bigg)^2 \Big(1 + 2\norm{P}_2\norm{L^{-1}}_2\Big)^2,
    \end{align*}
    where the last step uses that $\norm{L}_F \leq \sqrt{n}\norm{L}_{2\to\infty}$.
    From this point all manipulations are identical to the previous case, and so we can derive the same guarantee.
\end{proof}
If we allow for the matrix $L$ to have negative values, then we can get a slightly weaker statement.
Although we do not consider this class of matrices $L$ for our applications, we state it as a separate lemma as it may be of independent interest, but only sketch a proof.
\begin{lemma}\label{lemma:perturbed_matrix_general}
  Let $LR$ be a factorization, and let $\hat{L}$ be an $(\eta, \mu)$-perturbation of $L\in\R^{n\times n}$ for $\eta = \frac{\xi}{144\norm{L}_F\norm{L^{-1}}_2}$ and $\mu = \frac{\xi\norm{L}_2}{144n\kappa(L)}$ where $0 \leq \xi \leq 24$.
  Also let $L$ and $\hat{L}$ be invertible, and define $\hat{R} = \hat{L}^{-1} LR$.
  Then $\mse(\hat{L},\hat{R})\leq (1+\xi)\mse(L, R)$ and $\maxerr(\hat{L},\hat{R})\leq (1+\xi)\maxerr(L, R)$.
\end{lemma}
\begin{proof}[Proof sketch]
    The proof is almost identical to that of \Cref{lemma:main_approximation}.
    We need an equivalent version of \Cref{lemma:bound_operator_norm_of_perturbation} where instead $\norm{P}_2 \leq \eta\norm{L}_F + \mu n$, where the relaxation of norm is a result of the operator norm being sensitive to the sign of matrix entries (unlike the Frobenius norm).
    We can derive such a bound from the proof of \Cref{lemma:good_frobenius} and noting that $\norm{L+P}_2 \leq \norm{L}_F + \norm{P}_2 \leq \norm{L}_F + \norm{P}_F$.
    This looser bound on $\norm{P}_2$ translates into a stricter bound on $\eta$, but otherwise the structure of the proof is the same.
    
    Similarly to the other proof, we assume $\eta \leq \frac{1}{6\norm{L}_F\norm{L^{-1}}_2}$ and $\mu \leq \frac{\norm{L}_2}{6n\kappa(L)}$.
    We have that $\norm{P}_2 \leq \norm{P}_F \leq \eta\norm{L}_F + \mu n \leq \frac{1}{3\norm{L^{-1}}_2} < \frac{1}{2\norm{L^{-1}}_2}$.
    From this point on, we can repeat the steps in the proof of \Cref{lemma:main_approximation} to bound the multiplicative blow-up in error.
\end{proof}

At this point we have derived conditions for when a small perturbation of $L$ into $\hat{L}$ induces a factorization $\hat{L}\hat{R}$ with good error.
A separate, perhaps surprising point, is that the lemmas we have derived make no assumptions on the structure of $R$.
No bounds on any norm of $R$ is necessary, and $R$ need not even be a square matrix.
Only the left factor $L$ determines the error bounds.

Our next point is to characterize when a left factor matrix is amenable to running in low space for streaming.

\subsection{Binned matrices are space (and time) efficient}
To the best of our knowledge, there is no universal definition for space efficiency in the context of real-valued vector inputs.
In the case where $L$ is a Toeplitz matrix with entries on its subdiagonals $(d_0, d_1, \dots, d_n)$, \cite{dvijotham2024efficient} identified that if these entries satisfy a recursion $d_k = \sum_{j=1}^{s} \lambda_j\cdot d_{k-j}$, for $k \geq s$, then $L\vec{z}$ can be computed in space $O(s)$.
They called matrices meeting this definition \emph{Buffered Linear Toeplitz matrices}.
Our class of space-efficient matrices, \emph{binned matrices} is distinct from theirs --- focusing instead on matrices with few unique entries per row -- and in particular it includes non-Toeplitz matrices.

To introduce our class of matrices we will rely on \emph{partitions of intervals}.
\begin{definition}\label{def:partition}
    Let $\mathcal{I}$ be the set of all intervals on $[1, \infty)$ and $a \leq b$ be positive integers.
    We then refer to $C_{[a, b]}\subset\mathcal{I}$ as a \emph{partition of the interval $[a, b]$} if (1) $\bigcup_{I\in C_{[a, b]}} I = [a, b]$, and, (2) there are no overlapping intervals in $C_{[a, b]}$, i.e., $\forall I, I'\in C_{[a,b]}, I\neq I' : I\cap I' = \emptyset$.
    We let $(C_{[a, b]})_i$ represent the $i$\textsuperscript{th} interval in the partition when ordered by their left end-point.
\end{definition}
In particular, we will be interested in sequences of partitions, \emph{binnings}, formally defined next.
\begin{definition}\label{def:binnings}
    Let $n$ be a positive integer and consider the sequence of partitions $\mathcal{B} = (C_{[1]}, C_{[2]}, \dots, C_{[n]})$ where $\mathcal{B}^i = C_{[i]}$.
    We call $\mathcal{B}$ a \emph{binning of $[n]$} if, for each $i\in [n-1]$, we have that for each $I\in \B^i$, there exists an $I'\in \B^{i+1}$ such that $I\subseteq I'$.
    We define the size of a binning, $|\B|$, to be the largest size of all its partitions, i.e.\ $|\B| = \max_{i\in [n]} |\B^i|$.
\end{definition}
An important property of binnings are that, given a binning $\B$ of $[n]$, then for every $i\in[n-1]$, $\B^{i+1}$ can be constructed by merging the intervals in $\B^i \cup [i+1, i+1]$.
It turns out that this property, enables space-efficiency when applied to lower-triangular matrices.
\begin{definition}[$\B$-binned matrices]\label{def:binned_matrix}
    Let $L\in\R^{n\times n}$ be a lower-triangular matrix, and let $\B$ be a binning of $[n]$.
    If for each $i\in[n]$, we have that for all $I\in \B^i$, $L_{i, j}$ assumes the same value for all $j\in I$, then we say $L$ is a $\B$-binned matrix.
\end{definition}
Intuitively, a lower-triangular matrix $L$ is \emph{$\B$-binned} if on any row $i$ it is constant on intervals as described by $\B^i$.
This fact also implies that a $\B$-binned matrix can always be stored in space $O(n\cdot \abs{\B})$.
By definition, being a binned matrix is not a strong condition, as the trivial binning $\B$ where $|\B^i| = i$ holds for any lower-triangular matrix.
It turns out to be a meaningful concept nonetheless, as the size of the binning is linked to space and time complexity.
\begin{lemma}\label{lemma:bins_good_space}
    A $\B$-binned matrix $L$ has space complexity $\leq \abs{\B}$.
\end{lemma}
\begin{lemma}\label{lemma:bins_good_time}
    Consider a $\B$-binned matrix $L\in\R^{n\times n}$ and $\vec{z}\in\R^n$.
    Then $L\vec{z}$ can be computed incrementally in time $O(\abs{\B})$ per output if, at the time of outputting $(L\vec{z})_i$, we have that:
    \begin{enumerate}
        \item The partitions $\B^{i-1}$ (for $i > 1$) and $\B^{i}$ are accessible and stored in sorted order.
        \item We have random access to the $i$\textsuperscript{th} row of $L$.
    \end{enumerate}
\end{lemma}
To prove Lemmas \ref{lemma:bins_good_space} and \ref{lemma:bins_good_time}, we begin by proving intermediate results for how the computation $L\vec{z}$ can be performed efficiently using an internal state of size $\abs{\B}$.
\begin{remark}\label{remark:space_row_by_row}
   Let $\B$ be a binning of $[n]$ and $\vec{z}\in\R^n$. 
   For any $i\in[n-1]$, define $\vec{b}^i\in\R^{\abs{\B}}$ where $\vec{b}^i_k = \sum_{j\in\B^i_k} \vec{z}_j$.
   Then, given only $\B^i$ and $\B^{i+1}$ in sorted order, vector $\vec{b}^i$ and scalar $\vec{z}_{i+1}$, we can compute $\vec{b}^{i+1}$ in time $O(|\B|)$.
   Furthermore, $\vec{b}^{i+1}$ can be computed in-place using the memory buffer of $\vec{b}^{i}$.
\end{remark}
\begin{proof}
    Since there is a bijection between $\vec{b}^i_k$ and a unique interval $I = \B^{i}_k$, we use the notation $\vec{b}^i_I$ when convenient.
    For every $I\in\B^{i+1}$ where $i+1\notin I$, we have by the definition of binnings that there exists a unique set of disjoint intervals $S\subseteq\B^{i}$ such that $\bigcup_{I'\in S} I' = I$.
    In particular this allows us to write $\vec{b}^{i+1}_I = \sum_{j\in I} \vec{z}_j = \sum_{I'\in S} \vec{b}^i_{I'}$.
    By the same argument, there must exist a set $S\subseteq\B^i$ such that for the unique interval $I\in\B^{i+1}$ containing $i+1$, we can express it as $\vec{b}^{i+1}_I = \vec{z}_{i+1} + \sum_{I'\in S} \vec{b}^i_{I'}$.
    The computation is therefore possible.
    
    For the efficiency of this computation, note further that any entry of $\vec{b}^{i+1}$ is a range sum on $\vec{b}^{i}$, with the exception of the entry $\vec{b}^{i+1}_{\abs{\B^{i+1}}}$ that is a (potentially empty) range sum on $\vec{b}^i$ plus $\vec{z}_{i+1}$.
    These ranges are disjoint by the definition of a binning.
    It is also the case that $\vec{b}^{i+1}_k$ will only be updated using entries $\vec{b}^{i}_{k'}$ where $k' \geq k$, due to how intervals in $\B^{i+1}$ are the result of merging intervals in $\B^i$.
    It follows that we can update $\vec{b}^{i}$ in-place into $\vec{b}^{i+1}$ by updating the entries in ascending order.
    It can also be done efficiently by maintaining two pointers for $\vec{b}^{i}$, one for reading and one for writing (the reading pointer always behind or at the same position as the writing pointer), and two reading pointers for $\B^i$ and $\B^{i+1}$.
    The update requires scanning each of these structures once and performing $O(1)$ constant time operations each time a pointer is moved.
    Each structure containing $O(\abs{\B})$ elements thus implies the time complexity.
\end{proof}
\begin{remark}\label{remark:compute_all_buffers}
    Using the setup in \Cref{remark:space_row_by_row}, consider the output sequence $\vec{b}^1, \dots, \vec{b}^{n}\in\R^{\abs{\B}}$.
    If at the time of updating $\vec{b}^{i}$ into $\vec{b}^{i+1}$ we have access to $\B^i$ and $\B^{i+1}$ in sorted order, we can compute the sequence in-place and in time $O(\abs{\B})$ per output.
\end{remark}
\begin{proof}
    The statement follows almost immediately from \Cref{remark:space_row_by_row} and induction.
    For the base case, $\vec{b}^1\in\R^{\abs{\B}}$ has a single non-zero entry equaling $\vec{z}_1$, and so can be computed in time $O(\abs{\B})$.
    Next, for the induction step, assume we have already computed $\vec{b}^i$ and have it stored in memory.
    By our assumption we have access to $\B^i$ and $\B^{i+1}$, allowing us to invoke \Cref{remark:space_row_by_row} to compute $\vec{b}^{i+1}$ using $\vec{b}^i$ in-place and in time $O(\abs{\B})$.
    Invoking induction thus proves the claim.
\end{proof}
We are now ready to prove Lemmas~\ref{lemma:bins_good_space}~and~\ref{lemma:bins_good_time}.
\begin{proof}[Proof of \Cref{lemma:bins_good_space}]
    Given $\vec{b}^i$, as defined in \Cref{remark:space_row_by_row}, and $L$, we have enough information to output $(L\vec{z})_i$.
    This follows from $L$ being constant on the intervals given in $\B$, meaning we can write $(L\vec{z})_i = \sum_{I\in\B^i} L_{i, I}\vec{b}^i_I$.
    Since $\B$ can be read directly from $L$, we can therefore invoke \Cref{remark:compute_all_buffers}, proving that an internal state of size $\abs{\B}$ is sufficient for computing $L\vec{z}$.
\end{proof}
\begin{proof}[Proof of \Cref{lemma:bins_good_time}]
    Analogously to the proof of \Cref{lemma:bins_good_time}, note that the $i$\textsuperscript{th} row of $L$ and $\vec{b}^i$ is sufficient to output $(L\vec{z})_i$.
    This computation, from $\vec{b}^i$ to $(L\vec{z})_i$, requires $O(\abs{\B})$ additions, multiplications, and random accesses into the $i$\textsuperscript{th} row of $L$ -- all constant time operations.
    Since the necessary partitions in $\B$ are provided to update $\vec{b}^i$ into $\vec{b}^{i+1}$ at each step, by the assumption in the lemma, we can also invoke \Cref{remark:compute_all_buffers}.
    We thus have that the internal state can be updated in time $O(\abs{\B})$, and that the state can be mapped to an output in time $O(\abs{\B})$, proving our lemma.
\end{proof}

\subsection{Perturbing matrices into binned matrices}

While binned matrices are nice, the issue is that most matrices do not admit very interesting binnings.
The Bennett matrix for example only admits the trivial binning.
Our intention is therefore to approximate a matrix $L$ by a matrix $\hat{L}$ that has a non-trivial binning, but is still \enquote{close} to $L$.
\begin{definition}\label{def:b_approximation}
    A lower-triangular matrix $L\in\R^{n\times n}$ is said to be \emph{$\B$-approximated} by a $\B$-binned matrix $\hat{L}\in\R^{n\times n}$ if for every $1\leq j \leq i \leq n$, we have that $\hat{L}_{i, j} = \frac{L_{i, a} + L_{i, b}}{2}$ where $[a, b]\in\B^i$ and $j\in[a, b]$.
\end{definition}
Intuitively, $\hat{L}$ is the result of superimposing the binning on the matrix $L$, and setting the entries of $\hat{L}$ to be equal to the average of the entries in $L$ at the end of each interval.
Note that any other rule for how to set the entries (including mixing rules across intervals) would also result in a $\B$-binned matrix.
A more natural rule than what is proposed in \Cref{def:b_approximation} would have been to choose $\hat{L}$ to be the \emph{average over all entries in each interval}, since this would minimize $\norm{\hat{L}-L}_F$.
This choice would however necessitate querying $\Omega(n)$ entries in~$L$ to produce a row in $\hat{L}$.
Our chosen rule approximates the average instead using $O(\abs{\B})$ entries.
This will prove important later when considering the time complexity of computing $\hat{L}\vec{z}$.

Note that \Cref{def:b_approximation} only specifies how a matrix can be approximated given a binning $\B$, but it says nothing about how to find a $\B$ that performs well.
Our next result is an algorithm, given as \Cref{alg:matrix_to_binning}, that produces a $\B$ with good guarantees for a certain class of matrices.
\begin{algorithm}[t]
   \caption{Our binning procedure}\label{alg:matrix_to_binning}
   \begin{algorithmic}[1]
   \STATE {\bfseries Input:} Matrix $L\in\R_{\geq 0}^{n\times n}$, positive scalars $c, \tau <1$
    \STATE Initialize sets $\B^0, \B^1, \dots, \B^n \gets \emptyset$
    \FOR{$i = 1$ {\bfseries to} $n$}
        \STATE $k \gets |\B^{i-1}|$
        \WHILE{$k > 0$}
            \STATE $[a, b] \gets \B^{i-1}_k$, $k' \gets k$
            \IF{$k > 1$ {\bfseries and} $L_{i, b+1} > 0$ {\bfseries and} $L_{i, a} / L_{i, b+1} > c$}
                \STATE $[a', b'] \gets \B^{i-1}_{k'-1}$
                \WHILE{$k' > 1$ {\bfseries and} $L_{i, a'} / L_{i, b+1} \geq c^2$}
                    \STATE $a \gets a'$, $k' \gets k' - 1$
                    \STATE $[a', b'] \gets \B^{i-1}_{k'-1}$
                \ENDWHILE
                \IF{$k' = 1$ {\bfseries and} $L_{i, a'} / L_{i, b+1} \geq c^2$}
                    \STATE $a \gets a'$
                \ENDIF
            \ENDIF
            \IF{$L_{i, b} \leq \tau$}
                \STATE Add $[1, b]$ to $\B^i$, $k \gets 0$
            \ELSE
                \STATE Add $[a, b]$ to $\B^i$, $k \gets k' - 1$
            \ENDIF
        \ENDWHILE
        \STATE Add $[i, i]$ to $\B^i$
        \STATE {\bfseries output} $\B^i$
    \ENDFOR
   \end{algorithmic}
\end{algorithm}
We describe the core idea of the algorithm next.
Given a lower-triangular matrix~$L$, and positive constants $c, \tau < 1$, it produces a binning $\B$ parameterized in $c, \tau$.
It tries to maintain that for any $i\in[n]$ and $[a, b]\in \B^i$, $b < i$, we have that $c^2 \leq \frac{L_{i, a}}{L_{i, b+1}} \leq c$, with the possible exception of a single (possibly large) interval containing only entries less than or equal to $\tau$.
It achieves this by greedily merging intervals from one row to the next.

Specifically, given the partition of row $i$, $\B^{i}$, it constructs the next row's partition, $\B^{i+1}$ by greedily merging intervals in the partition $\B^{i}$, and then adding the interval $[i+1,i+1]$ to the result.
Iterating over the intervals in $\B^i$, in order of decreasing column indices, it greedily merges two intervals $[a, b], [b+1, d]$ whenever $\frac{L_{i+1, b+1}}{L_{i+1,d+1}} > c$ and the resulting interval $[a, d]$ satisfies $\frac{L_{i+1, a}}{L_{i+1, d+1}} \geq c^2$.
As long as the result is still $\geq c^2$, it continues merging with new intervals.
Once a candidate interval $[a, b]$ with $L_{i, b} \leq \tau$ is encountered, it merges all remaining intervals in $\B^i$ into a single interval $[1, b]$.
Since the definition of a binning can be expressed in terms of merging intervals in $\B^i \cup \{[i+1, i+1]\}$ to produce $\B^{i+1}$, the procedure will always produce a binning, potentially the trivial binning.

We begin by showing that, under an assumption of random access to entries in $L$, the $\B$-approximation of $L$ derived from running \Cref{alg:matrix_to_binning}, $\hat{L}$, has good time complexity.
\begin{lemma}\label{lemma:alg_works_for_time}
    Consider a lower triangular matrix $L$ and let $\hat{L}$ be the $\B$-approximation of $L$ where $\B$ is the output from running \Cref{alg:matrix_to_binning} on $L$.
    Given $\vec{z}\in\R^n$ and random access to entries in $L$, then $\hat{L}\vec{z}$ can be computed row-by-row, reading one element of $\vec{z}$ at a time, in time $O(\abs{\B})$ per output.
\end{lemma}
\begin{proof}
    Our approach will rely on computing $\B$ and $\hat{L}\vec{z}$ in parallel.
    First observe that the for-loop in \Cref{alg:matrix_to_binning} runs in time $O(\abs{\B})$.
    This can be seen by noting that any given interval in $\B^{i-1}$ appears at most twice inside the outermost while-loop: once as a possible candidate for merging into on line~11, and/or as the interval to start merging from on line~8.
    Since all operations inside the for-loop can be made constant time operations (accessing an entry of $L$ is constant time due to the random access assumption) the time to output $\B^{i}$ given $\B^{i-1}$ is $O(\abs{\B})$.

    Keeping both $\B^{i-1}$ and $\B^i$ in memory when outputting $(\hat{L}\vec{z})_i$ almost allows us to invoke \Cref{lemma:bins_good_time} to prove the statement, but we still need random access to the $i$\textsuperscript{th} row of $\hat{L}$.
    We can simulate this random access by choosing our approximation rule in \Cref{def:b_approximation} to only depend on a constant number of elements in $L$.
    In particular, picking the stated rule where for $[a,b]\in\B^i : \hat{L}_{i, [a, b]} = \frac{L_{i, a} + L_{i, b}}{2}$ works, concluding the proof.
\end{proof}
At this point, we have shown that $\B$-binned matrices derived from running \Cref{alg:matrix_to_binning} have space and time complexity linear in the binning size.
We show next that, for a certain class of matrices, \Cref{alg:matrix_to_binning} can produce binnings achieving a good trade-off between binning size and error when applied to the left factor in a factorization.
\begin{definition}[MRMs]\label{def:nice_matrix}
    We say that an invertible lower-triangular matrix $L\in\R^{n\times n}$ is a \emph{monotone ratio} matrix (\emph{MRM} for short) if it satisfies that
    \begin{enumerate}
        \item All entries $0 < L_{i, j} \leq 1$ for all $i, j\in [n]$ with $i \leq j$.
        \footnote{All results for MRMs would still hold if entries are allowed to be zero, but at the cost of a more convoluted property (3).}
        \item For all $i\in[n]$ the function $g(x) = L_{i, x}$ is non-decreasing for $x\in[i]$.
        \item For all $j, j'\in [n]$ where $j \leq j'$ the function $h(x) = \frac{L_{x, j}}{L_{x, j'}}$ is non-decreasing for $x\in [j', n]$.
    \end{enumerate}
    When an MRM $L$ is the left factor of a factorization $LR$, we call the factorization an \emph{MRM factorization}.
\end{definition}
Before stating how our main algorithm performs on these matrices, we briefly note that the requirement that the entries of the matrix are less than or equal to 1 is without loss of generality when applied to factorizations:
For every scalar $a > 0$, $LR = (\frac{1}{a}L)(aR) = L' R'$ yields a new factorization with all entries re-scaled.
\begin{lemma}\label{lemma:alg_works_for_space}
    Executing \Cref{alg:matrix_to_binning} on an MRM $L\in\R^{n\times n}$ with positive $c, \tau < 1$ returns a $\B$ with $|\B| = O\Big(\frac{\log(1/\tau)}{\log(1/c)}\Big)$.
\end{lemma}
\begin{proof}
    We will argue for an arbitrary row index $i\in[n]$.
    Note that the two left-most intervals in $\B^i$ are the only possible intervals $I\in\B^i$ where there could be $j\in I : L_{i, j} \leq \tau$ due to the condition on line~17.
    We will argue that the remaining intervals cannot be too many.
    Define $S\subseteq \B^i$ as $\B^i$ with these unique intervals removed if they exist, where now for all $[a, b] \in S : \tau < L_{i, a} \leq 1$.
    For any two contiguous intervals $[a, b], [b+1, \ell] \in S$, at least one of the two ratios $\frac{L_{i, a}}{L_{i,b+1}}$ and $\frac{L_{i, b+1}}{L_{i, \ell+1}}$ must be $\leq c$.
    To prove this, assume to the contrary that both are $> c$.
    Then the merged interval $[a, \ell]$ would have a ratio of $\frac{L_{i, a}}{L_{i, \ell+1}} = \frac{L_{i, a}}{L_{i,b+1}}\cdot\frac{L_{i, b+1}}{L_{i,\ell+1}} \geq c^2$.
    Since both intervals have a ratio $> c$ and merging them would yield a ratio $\geq c^2$, these intervals would have gotten merged during the execution of \Cref{alg:matrix_to_binning}, proving the claim by contradiction.
    We thus have the property that for any row index $i$, at least $1/2 - o(1)$ of these intervals in $S$ satisfy having a ratio $\leq c$.
    Since $L$ is an MRM, we have that entries on each row of $L$ are non-decreasing with the column index, therefore the ratio can only decrease by a factor $c$ at most $O\big(\log_c(1/\tau)\big) = O\big(\frac{\log(1/\tau)}{\log(1/c)}\big)$ times, implying $|\B| \leq \abs{S} + 2  = O\big(\frac{\log(1/\tau)}{\log(1/c)}\big)$.
\end{proof}
Having shown that the size of the binning can be parameterized in $c, \tau$, we prove next that the induced matrix approximation is a $(1/c^2-1, \tau)$-perturbation of the original one.
\begin{lemma}\label{lemma:alg_produces_good_perturbation}
    Consider an MRM $L$ and let $\hat{L}$ be the $\B$-approximation of $L$ where $\B$ is the output from running \Cref{alg:matrix_to_binning} on $L$ with positive $ c, \tau < 1$.
    Then $\hat{L}$ is invertible and a $(1/c^2-1, \tau)$-perturbation of $L$.
\end{lemma}
\begin{proof}
    We start by arguing for the case where $\tau < \min_{i \geq j} L_{i, j}$.
    In this case the condition on line~17 is never true, and so all intervals produced are solely based on $c$.
    $\B$ then satisfies that: for all $i\in[n]$ and $[a, b]\in\B^i$, $\frac{L_{i, a}}{L_{i, b}} \geq c^2$.
    For $a = b$, the statement is immediate, whereas for $a < b$ we note that the interval must have been the result of a merge at some row index $i'\leq i$.
    At the time of the merge, it must have satisfied $c^2 \leq \frac{L_{i', a}}{L_{i', b+1}}$.
    Combining the second and third property of \Cref{def:nice_matrix}, we therefore have that $c^2 \leq \frac{L_{i', a}}{L_{i', b+1}} \leq \frac{L_{i, a}}{L_{i, b+1}} \leq \frac{L_{i, a}}{L_{i, b}}$, proving the statement in general.

    We next wish to prove that for all $i, j\in[n]$ we have $ c^2 L_{i,j} \leq \hat{L}_{i,j} \leq \frac{1}{c^2} L_{i, j}$.
    For $i < j$, the statement is trivial as $L_{i, j} = \hat{L}_{i, j} = 0$, so consider $i\geq j$.
    Assuming $j\in[a, b]\in\B^i$, then we have that $\hat{L}_{i, j} = \frac{L_{i, a} + L_{i, b}}{2}$ and thus $c^2 \leq \frac{L_{i, a}}{L_{i, b}} \leq \frac{\hat{L}_{i, j}}{L_{i, j}}$ due to the monotonicity of the rows of $L$ (Property 2 in \Cref{def:nice_matrix}).
    For the second inequality we get $\frac{1}{c^2} \geq \frac{L_{i, b}}{L_{i, a}} \geq \frac{\hat{L}_{i, j}}{L_{i, j}}$ by an identical argument.
    Letting $P = \hat{L} - L$, we therefore have that for every $i, j\in[n]$ that $P_{i,j} \leq (1/c^2 -1) L_{i, j}$ and $P_{i, j} \geq (c^2 - 1) L_{i, j}$, thus we get that $|P_{i,j}| \leq \max(1/c^2-1, 1-c^2) |L_{i,j}| = (1/c^2 - 1)|L_{i,j}|$.

    Now consider the case where $\tau \geq \min_{i \geq j} L_{i, j}$ and line~17 evaluates to true for a row $i\in[n]$.
    In this case every interval except the left-most one of the form $[1, b]$ is guaranteed to satisfy the same condition.
    However, since $\hat{L}_{i, [1, b]}$ will be set to the average of two positive values less than or equal to $\tau$, we instead get that $\abs{P_{i, j}} = \abs{\hat{L}_{i, j} - L_{i, j}} \leq \tau$ for $j\in[1, b]$.
    Thus $P$ satisfies $\abs{P_{i, j}} \leq (1/c^2 - 1)\abs{L_{i,j}} + \tau$ for every $i, j\in [n]$.
    
    Lastly, as invertibility of a lower-triangular matrix is decided by its diagonal, we also note that $\hat{L}$ is invertible, since the $\B$ produced by \Cref{alg:matrix_to_binning} satisfies $[i,i]\in\B^i$, and therefore $\hat{L}_{i,i}=L_{i,i}$.
\end{proof}
\begin{lemma}\label{lemma:alg_works_for_error}
    Consider the MRM factorization $LR$ and let $\hat{L}$ be the $\B$-approximation of $L\in\R^{n\times n}$ where $\B$ is the output from running \Cref{alg:matrix_to_binning} on $L$ for $c = \exp(-\frac{\xi}{576\kappa(L)})$ and $\tau = \frac{\xi\norm{L}_2}{144n\kappa(L)}$ for $0 < \xi \leq 24$.
    Then $\hat{L}$ induces a factorization $\hat{L}\hat{R} = LR$ with error guarantees $\mse(\hat{L},\hat{R}) \leq (1+\xi)\mse(L, R)$ and $\maxerr(\hat{L},\hat{R}) \leq (1+\xi)\maxerr(L, R)$.
\end{lemma}
\begin{proof}
    From \Cref{lemma:alg_produces_good_perturbation}, we have that $\hat{L}$ is a $(1/c^2 -1, \tau)$-perturbation of $L$.
    The statement to prove follows immediately from \Cref{lemma:main_approximation} if we show that the conditions imposed on $1/c^2 - 1$ and $\tau$ are stricter or equal those from \Cref{lemma:main_approximation}.
    $\mu$ is set here to the same value as in that lemma, so the only consideration is $\eta$.
    Letting $c = \exp(-1/s)$, we would like to set $s$ such that $1/c^2 - 1 = \exp(2/s) - 1 \leq 4/s = \eta = \frac{\xi}{144\kappa(L)}$, where the inequality is true if $s \geq 2$.
    Solving for $s$ in the last step, we get $s = \frac{576\kappa(L)}{\xi} \geq \frac{576\cdot 1}{24} = 24$, and so the inequality is valid.
    Since we set $\tau$ equal to $\mu$, and $1/c^2 - 1$ smaller than or equal to the $\eta$ in \Cref{lemma:main_approximation}, we are guaranteed the same approximation factor.
\end{proof}

We are now ready to state the guarantees of our binning procedure when applied to MRMs.
\begin{theorem}\label{thm:binning_approx}
    Consider the MRM factorization $LR$ and let $\hat{L}$ be the $\B$-approximation of $L\in\R^{n\times n}$ where $\B$ is the output from running \Cref{alg:matrix_to_binning} on $L$ for $c = \exp(-\frac{\xi}{576\kappa(L)})$ and $\tau = \frac{\xi\norm{L}_2}{144n\kappa(L)}$ for $0 < \xi \leq 24$.
    Then $\abs{\B} = O_\xi(\kappa(L)\log(n\norm{L^{-1}}_2))$ and $\hat{L}$ induces a factorization $\hat{L}\hat{R} = LR$ where:
    \begin{itemize}
        \item $\hat{L}\hat{R}$ satisfies $\mse(\hat{L},\hat{R}) \leq (1+\xi)\mse(L,R)$ and $\maxerr(\hat{L},\hat{R}) \leq (1+\xi)\maxerr(L,R)$.
        \item $\hat{L}$ has space complexity at most $\abs{\B}$.
        \item Given $\vec{z}\in\R^n$ and random access to $L$, $\hat{L}\vec{z}$ can be computed incrementally in time $O(\abs{\B})$ per output.
    \end{itemize}
\end{theorem}
\begin{proof}
    Using the setup in \Cref{lemma:alg_works_for_error} and then invoking \Cref{lemma:alg_works_for_space} for the given values of $c$ and $\tau$ gives $\abs{\B} = O(\kappa(L)\log(n\norm{L^{-1}}_2 /\xi) / \xi)$ for the error guarantees.
    This can then be translated into the statements on space and time complexity by invoking \Cref{lemma:bins_good_space} and \Cref{lemma:alg_works_for_time} respectively.
\end{proof}
We reiterate a previous point from \Cref{lemma:main_approximation}: the guarantees in \Cref{thm:binning_approx} are \emph{independent} of $R$.
The right factor $R$ can be chosen arbitrarily as long as it has $n$ rows.
It is only properties of the left factor $L$ that decide the approximation's guarantees.

\subsection{Applications to square-root factorizations for counting}
While we have now shown that our approach for binning works well for certain factorizations, we show next that this class includes the square-root factorization for the counting matrix.
Calling back to applications, the counting matrix receives as much attention as it does, in part, due to its role in facilitating private learning.
Recalling that $n$ steps of weight updates in gradient descent with \emph{momentum} and \emph{weight decay} can be expressed as
\begin{equation*}
    \vec{w}^i = \alpha\vec{w}^{i-1} - \eta \vec{m}^i\ \mathrm{ given }\ \vec{m}^i = \beta \vec{m}^{i-1} + \vec{g}^i
\end{equation*}
where $\vec{w}^0 = \vec{0}, \vec{w}^1, \dots, \vec{w}^n\in\R^d$ are weight vectors, $\vec{g}^1, \dots, \vec{g}^n\in\R^d$ are gradient updates, $\eta > 0$ is the learning rate, $0\leq \beta < 1$ is the momentum strength, and $0 < \alpha\leq 1$ is the weight decay.
Unrolling the recursion, we can express the weight updates in matrix form as:
\begin{equation*}
    W = \begin{bmatrix}
    (\vec{w}^1)^T\\
    \vdots\\
    (\vec{w}^n)^T
\end{bmatrix},
    G = \begin{bmatrix}
    (\vec{g}^1)^T\\
    \vdots\\
    (\vec{g}^n)^T
\end{bmatrix} :
    W = \eta\cdot A_{\alpha, \beta} G\enspace ,
\end{equation*}
where $A_{\alpha, \beta}\in\R^{n\times n}$ is a lower-triangular Toeplitz matrix with subdiagonals $a_0, \dots, a_{n-1}$ of the form $a_k = \frac{\alpha^{k+1} - \beta^{k+1}}{\alpha - \beta}$.
A factorization mechanisms for $A_{\alpha, \beta}$ thus implies a primitive for private learning (privatizing weight updates).
We restrict $\alpha$ and $\beta$ to satisfy $0 \leq \beta < \alpha \leq 1$ with the motivation that these ranges make sense when $\alpha, \beta$ are interpreted as weight decay and momentum respectively. %
In the particular case of $A_{\alpha=1, \beta=0}=A$, we recover the standard all-1s lower-triangular counting matrix.

It turns out that $A_{\alpha, \beta}$ has a closed form expression for its square-root for general $\alpha, \beta$:
\begin{fact}[Theorem~1 and Lemma~6 in \cite{KalLamp24}]\label{fact:squre_root_factorization}
   Let $B_{\alpha, \beta}\in\R^{n\times n}$ where $A_{\alpha, \beta} = B_{\alpha, \beta}^2$. Then $B_{\alpha, \beta}$ is a lower-triangular Toeplitz matrix with nonincreasing subdiagonals $b_0, \dots, b_{n-1}$ where $b_j = \sum_{i=0}^j \alpha^{j-i}\gamma(j-i)\gamma(i)\beta^i$, $\gamma(k) = \frac{1}{4^k}\binom{2k}{k}$, and
    \begin{equation*}
        \frac{\alpha^j}{2\sqrt{j+1}} \leq b_j
        \leq \frac{\alpha^j}{(1- \beta/\alpha)\sqrt{j+1}}\,.
    \end{equation*}
\end{fact}
The square-root factorization is interesting to study as it known to give asymptotically optimal error \cite[Theorem 3 and 4]{KalLamp24}.
In the particular case of the Bennett matrix $B = \sqrt{A} = B_{\alpha=1, \beta=0}$, it is well-studied \cite{bennett77,fichtenberger_constant_2023,henzinger_almost_2023,dvijotham2024efficient}, and gives the optimal leading constant over all factorizations for both mean and maximum squared error.
The only issue with the factorization $L=R=B_{\alpha, \beta}$ is that the best known bound on the space complexity of $B_{\alpha, \beta}$ is $\Omega(n)$, for all choices of $\alpha, \beta$.
We show next that our binning approach can approximate the square-root factorization with space complexity $o(n)$, and begin by proving that $B_{\alpha, \beta}$ is an MRM.
\begin{fact}\label{fact:sqrt_is_mrm}
   $B_{\alpha, \beta}\in\R^{n\times n}$ is an MRM. 
\end{fact}
\begin{proof}
    $B_{\alpha, \beta}$ trivially satisfies the first two properties in \Cref{def:nice_matrix} by noticing that $b_j = \Theta_{\alpha,\beta}(\alpha^j/\sqrt{j})$ according to \Cref{fact:squre_root_factorization}.
    What remains to show is that it has a monotone ratio.
    For all $j,j'\in[n]$, where $j < j'$, we consider the function $h(x) = \frac{(B_{\alpha, \beta})_{x, j}}{(B_{\alpha, \beta})_{x,j'}} = \frac{b_{x-j}}{b_{x-j'}} = \Theta_{\alpha, \beta}(\alpha^{j-j'}\sqrt{\frac{x-j'}{x-j}}) = \Theta_{\alpha, \beta}(\alpha^{j-j'}\sqrt{1 + \frac{j-j'}{x-j}})$ on $x\in[j', n]$.
    Since $j < j'$, $j-j'$ is negative and thus $h(x)$ is increasing with $x$, proving the last property.
\end{proof}

To derive the quality of our approximation, we need to upper bound $\norm{B_{\alpha,\beta}}_2$ and $\norm{B_{\alpha, \beta}^{-1}}_2$.
To this end, we will use the following lemma for the operator norm of lower-triangular Toeplitz matrices.
\begin{lemma}\label{lemma:operator_norm_lttoeplitz}
    Let $Q\in\R^{n\times n}$ be a lower-triangular Toeplitz matrix with subdiagonals $q_0, \dots, q_{n-1}$.
    Then $\norm{Q}_2 \leq \sum_{i=0}^{n-1} \lvert q_i\rvert$.
\end{lemma}
\begin{proof}
    Let $Q^{(i)}$ denote the matrix which is zero everywhere except on the $i$\textsuperscript{th} subdiagonal where it equals $q_i$.
    We therefore have that, $\norm{Q}_2 = \norm{\sum_{i=0}^{n-1} Q^{(i)}}_2 \leq \sum_{i=0}^{n-1}\norm{Q^{(i)}}_2 = \sum_{i=0}^{n-1} \abs{q_i}$, where the first inequality is repeatedly using the triangle inequality, and the last equality comes from identifying that a matrix consisting of only a subdiagonal of 1s has an operator norm of 1. 
\end{proof}
For the bound on $\norm{B_{\alpha, \beta}^{-1}}_2$, we will also need the following bound.
\begin{lemma}\label{lemma:gamma_diff}
    For integer $k\geq 0$, define $\gamma'(k)$ as the function where $\gamma'(0) = 1$ and $\gamma'(k) = \gamma(k) - \gamma(k-1)$ for $k \geq 1$.
    Then for $k \geq 1$,
    \begin{equation*}
        -\frac{1}{\sqrt{\pi k}(2k-1)}\leq \gamma'(k) \leq -\frac{1}{2\sqrt{k}(2k-1)} \,.
    \end{equation*}
\end{lemma}
\begin{proof}
    Note that $\gamma(k) - \gamma(k-1) = \frac{1}{4^k}(\binom{2k}{k} - \binom{2k-2}{k-1}) = \frac{\binom{2k}{k}}{4^k}(1 - \frac{4k^2}{2k(2k-1)}) = -\frac{1}{2k-1}\gamma(k)$.
    Invoking known bounds on $\gamma : \frac{1}{2\sqrt{k}} \leq \gamma(k) \leq \frac{1}{\sqrt{\pi k}}$ \cite[Lemma 5]{KalLamp24}, proves the claim .
\end{proof}
We are now ready to prove the operator norm bounds.
\begin{lemma}\label{lemma:inv_good_operator_norm}
    For $0 \leq \beta < \alpha \leq 1$ and integer $n > 1$, $B_{\alpha, \beta}\in\R^{n\times n}$ has an inverse satisfying $\norm{B_{\alpha, \beta}^{-1}}_2 \leq \frac{1}{1 - \beta/\alpha}C$, where $C$ is a constant independent of $\alpha, \beta$ and $n$.
\end{lemma}
\begin{proof}
    Define $E_\omega\in\R^{n\times n}$, $0 < \omega \leq 1$ to be the lower-triangular Toeplitz matrix with entries $(E_\omega)_{i, j} = \omega^{i-j}$ for $i\geq j$, with inverse $E_\omega^{-1}$ that is a matrix with 1s on the main diagonal and $-\omega$ on the subdiagonal immediately below.
    Since $A_{\alpha, \beta} = E_\alpha E_\beta$ (by inspection), we can write $\norm{B_{\alpha, \beta}^{-1}}_2 = \norm{A_{\alpha, \beta}^{-1} B_{\alpha, \beta}}_2 = \norm{E_\beta^{-1}E_{\alpha}^{-1} B_{\alpha, \beta}}_2$.
    Noting that $E_\alpha^{-1}B_{\alpha, \beta} = Q$ is a Toeplitz matrix with 1s on the main diagonal and subdiagonals $q_k = b_k - \alpha b_{k-1}$, we can write
    \begin{align*}
       q_k &= b_k - \alpha b_{k-1}\\
       &=\sum_{i=0}^k\alpha^{k-i}\gamma(k-i)\gamma(i)\beta^i - \sum_{i=0}^{k-1} \alpha^{k-i}\gamma(k-1-i)\gamma(i)\beta^i\\
       &= \beta^k\gamma(k) + \sum_{i=0}^{k-1}\alpha^{k-i}\gamma(i)\big(\gamma(k-i)-\gamma(k-i-1)\big)\beta^i\\ 
       &= \beta^k\gamma(k) + \sum_{i=0}^{k-1}\alpha^{k-i}\gamma(i)\gamma'(k-i)\beta^i\\ 
       &= \sum_{i=0}^{k}\alpha^{k-i}\gamma(i)\gamma'(k-i)\beta^i \enspace .
    \end{align*}
    Repeating the manipulation, we have that $S=E_\beta^{-1} Q$ is another Toeplitz matrix with 1s on the main diagonal and subdiagonals $s_k = q_k - \beta q_{k-1}$, where
    \begin{align*}
       s_k &= q_k - \beta q_{k-1}\\
       &= \sum_{i=0}^{k}\alpha^{k-i}\gamma(i)\gamma'(k-i)\beta^i\\
       &- \sum_{i=0}^{k-1}\alpha^{k-i-1}\gamma(i)\gamma'(k-i-1)\beta^{i+1}\\
       &= (\alpha^k + \beta^k)\gamma'(k) + \sum_{i=1}^{k-1}\alpha^{k-i}\gamma'(i)\gamma'(k-i)\beta^i\\
       &= \alpha^k \sum_{i=0}^{k}\bigg(\frac{\beta}{\alpha}\bigg)^i \gamma'(i)\gamma'(k-i) \enspace .
    \end{align*}
    In the end we are interested in bounding $\abs{s_k}$, and so we want to upper bound $f(i) = \abs{\gamma'(i)\gamma'(k-i)}$ for $i=0, \dots, k$.
    Noting that $f(i) = f(k-i)$ and $\abs{\gamma'(i)} = \Theta(i^{-3/2})$ (\Cref{lemma:gamma_diff}), we can bound $f(i)$ as $f(i) \leq f(0) = \abs{\gamma'(0)\gamma'(k)} = \abs{\gamma'(k)}$.
    Continuing, we therefore have
    \begin{equation*}
        \abs{s_k} \leq \alpha^k \abs{\gamma'(k)} \sum_{i=0}^{k}\bigg(\frac{\beta}{\alpha}\bigg)^i \leq  \frac{\alpha^k \abs{\gamma'(k)}}{1-\beta/\alpha}\,. %
    \end{equation*}
    Noting that $S = B^{-1}_{\alpha, \beta}$, we apply \Cref{lemma:operator_norm_lttoeplitz} to $S$:
    \begin{equation*}
        \norm{S}_2
        \leq \sum_{k=0}^{n-1} \abs{s_k}
        \leq \frac{1}{1 - \beta/\alpha}\bigg(1 + \sum_{k=1}^{n-1} \frac{\alpha^k}{\sqrt{\pi k}(2k-1)}\bigg)\,,
    \end{equation*}
    where the second inequality uses \Cref{lemma:gamma_diff}.
    For all valid choices of $\alpha, \beta$ the final sum can be shown to converge to a constant via e.g.\ bounding $\alpha\leq 1$ and using an integral approximation, proving the statement.
\end{proof}

\begin{lemma}\label{lemma:good_operator_norm}
    For $0 \leq \beta < \alpha \leq 1$ and integer $n > 1$, $B_{\alpha, \beta}\in\R^{n\times n}$ satisfies $\norm{B_{\alpha, \beta}}_2 \leq \frac{2\sqrt{n}-1}{1-\beta}$ for $\alpha=1$, and otherwise $\norm{B_{\alpha, \beta}}_2 \leq \frac{1}{(1-\beta/\alpha)(1-\alpha)}$.
\end{lemma}
\begin{proof}
    Invoking \Cref{lemma:operator_norm_lttoeplitz} we get that
    \[\norm{B_{\alpha, \beta}}_2 \leq \sum_{k=0}^{n-1} \abs{b_k} \leq \frac{1}{1-\beta/\alpha}\sum_{k=0}^{n-1} \frac{\alpha^k}{\sqrt{k+1}}\] 
    where the bound on $b_k$ comes from \Cref{fact:squre_root_factorization}.
    For $\alpha = 1$ we have 
    \[\frac{1}{\sqrt{k+1}} \leq \int_k^{k+1}\frac{1}{\sqrt{x}}\mathsf{d}x\]
    for $k>0$ and thus
    \begin{alignat*}{1}
        \norm{B_{\alpha, \beta}}_2 & \leq \frac{1}{1-\beta}(1 + \int_1^n \frac{1}{\sqrt{x}}\mathsf{d}x)\\
        & = \frac{1}{1-\beta}(1 + 2\sqrt{n} -2) 
         = \frac{2\sqrt{n}-1}{1-\beta} \enspace .
    \end{alignat*}
    In the case where $\alpha < 1$ we can write \[\norm{B_{\alpha, \beta}}_2 \leq \frac{1}{1-\beta/\alpha}\sum_{k=0}^{n-1} \alpha^j \leq \frac{1}{(1-\beta/\alpha)(1-\alpha)} \enspace .\]
\end{proof}

We can now state our approximation guarantee.%
\begin{theorem}\label{thm:sqrt_approx_guarantee}
    Let $0\leq \beta < \alpha \leq 1 $ and $\xi\in (0, 1)$.
    Further, let $\hat{L}$ be the $\B$-approximation of $B_{\alpha, \beta}\in\R^{n\times n}$ derived from running \Cref{alg:matrix_to_binning} with parameters $c$ and $\tau$.
    Then there exists $c, \tau$ for which
    \begin{equation*}
        \abs{\B} =
        \begin{cases}
            O_{\alpha, \beta, \xi}(\sqrt{n}\log n ) &\text{ if } \alpha=1,\\
            O_{\alpha, \beta, \xi}(\log n ) &\text{ otherwise},
        \end{cases}
    \end{equation*}
    and $\hat{L}$ induces a factorization $\hat{L}\hat{R} = B_{\alpha, \beta}^2$ where:
    \begin{itemize}
        \item $\mse(\hat{L},\hat{R}) \leq (1+\xi)\mse(B_{\alpha, \beta},B_{\alpha, \beta})$ and $\maxerr(\hat{L},\hat{R}) \leq (1+\xi)\maxerr(B_{\alpha, \beta},B_{\alpha, \beta})$.
        \item $\hat{L}$ has space complexity at most $\abs{\B}$.
        \item Given random access to $B_{\alpha, \beta}$, and $z\in\R^n$, $\hat{L}\vec{z}$ can be computed sequentially in time $O(\abs{\B})$ per output.
    \end{itemize}
\end{theorem}
\begin{proof}
    Noting that $B_{\alpha, \beta}$ is an MRM (\Cref{fact:sqrt_is_mrm}), the theorem follows immediately from invoking \Cref{thm:binning_approx} applied to the factorization $L=R=B_{\alpha, \beta}$, and using \Cref{lemma:inv_good_operator_norm} and \ref{lemma:good_operator_norm} for the bounds on the operator norms.
    For $\alpha = 1$ we arrive at a space complexity of $O\big(\frac{1}{(1-\beta)^2}\sqrt{n}\log(\frac{n}{\xi(1-\beta)}\big)/\xi\big)$
    \footnote{If $\alpha=1$, then it is possible to remove the $\xi$ inside of the logarithm by choosing $\tau$ smaller than the smallest entry in $B_{\alpha, \beta}$ and doing a finer analysis.}
    and $O\big(\frac{1}{(1-\beta/\alpha)^2(1-\alpha)}\log\big(\frac{n}{\xi(1-\beta/\alpha)}\big)/\xi\big)$ for $\alpha<1$.
\end{proof}
\Cref{thm:main_bennett} is a special case of \Cref{thm:sqrt_approx_guarantee} for $\alpha=1, \beta=0$.

\section{Empirical evaluation}\label{sec:empirical}
While we have explicit bounds on the space and error of our approach when applied to counting matrices (with and without momentum and/or weight decay), the strength of our approach is most easily visible in our empirical results.
For the canonical problem of continual counting (corresponding to $\alpha=1, \beta=0$) where the square-root factorization is asymptotically optimal up to lower order terms, \emph{we outperform it in sublinear space for practical regimes of $n$ for both mean squared and maximum squared error}.
This is, to the best of our knowledge, the first result explicitly making this point.
For the generalized counting problem including momentum and/or weight-decay, we approximate the square-root factorization of $A_{\alpha, \beta}$ with equal success.

\subsection{Notes on implementation and hardware}
We have implemented \Cref{alg:matrix_to_binning} together with helper methods in Python which allows for efficient computation of factorizations $\hat{L}, \hat{R}$, the exact mean and maximum squared error of which is shown.
Our implementation is available on \url{https://github.com/jodander/streaming-via-binning}.
All plots shown have been produced on a Macbook Pro 2021 with Apple M1 Pro chip and 16 GB memory, running Python 3.9.6, with scipy version 1.9.2, pandas version 1.5.3, and numpy version 1.23.3.
Although further optimization is possible, it can produce factorizations of $A\in\R^{n\times n}$ for $n$ up to $\sim 10^4$ in a few seconds.
The bottleneck of the experiments lies in storing all matrices involved as dense matrices.
As pointed out earlier in connection with \Cref{def:binned_matrix}, a $\B$-binned matrix can be stored in $O(n\cdot \abs{\B})$ memory, and the square-root factorizations we compare to can be stored in $O(n)$ memory due to their Toeplitz structure.

For all plots shown, we start from a factorization $L, R\in\R^{n\times n}$.
We $\B$-approximate $L$ by $\hat{L}$ using the binning $\B$ from running \Cref{alg:matrix_to_binning} on $L$ for positive parameters $c,\tau < 1$, then \emph{exactly} compute $\hat{R} = \hat{L}^{-1}LR$.
The $\B$-approximation we use is the one from \Cref{def:b_approximation}, averaging based on the end points of the intervals.
The space complexity of $\hat{L}$ will be equal to $|\B|$ and the error we report will be an exact computation based on $\hat{L}$ and $\hat{R}$, either $\mse(\hat{L}, \hat{R})$ or $\maxerr(\hat{L},\hat{R})$ depending on the experiment.
We make no use of any theoretical bounds.

\subsection{Results for the Bennett matrix}
To visualize the result, and showcase that our approach is meaningful already for small $n$, we give example factorizations in \Cref{fig:small_example}.
\begin{figure}[t]
\centering
\subfigure[Matrix $B$.]{\includegraphics{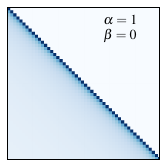}\label{fig:bennett_matrix}}
\subfigure[Left factor $\hat{L}$.]{\includegraphics{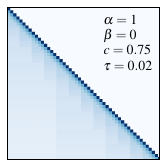}\label{fig:L_matrix}}
\subfigure[Right factor $\hat{R}$.]{\includegraphics{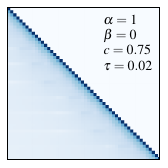}\label{fig:R_matrix}}
\caption{
    Plot of entries of different factorizations of the counting matrix $A = A_{\alpha=1, \beta=0}$ for $n=50$.
    Darker shades imply matrix entries closer to 1; white implies 0.
    \Cref{fig:bennett_matrix} corresponds to the factorization $A=B^2$.
    $\hat{L}$ is the $\B$-approximation of $B$ generated from \Cref{alg:matrix_to_binning} when run with $c=0.75, \tau=0.02$ on $B$, and $\hat{R} = \hat{L}^{-1} A$.
    For the matrices shown we have that $\mse(\hat{L}, \hat{R}) / \mse(B,B) = 0.9965$, $\maxerr(\hat{L}, \hat{R}) / \maxerr(B, B) = 0.9951$ and $\abs{\B} = 8$.
}\label{fig:small_example}
\end{figure}
The space usage of $\hat{L}$ is clearly visible in \Cref{fig:L_matrix}, since each row can be identified to have at most 8 unique colors, i.e.\ intervals.
As our technique relies on heuristics for perturbing a left-factor matrix into becoming space-efficient, and then bounding the blow-up of $\norm{\hat{R}}_{1\to 2}$, there is less to say about the structure in \Cref{fig:R_matrix}.
At most we can note qualitatively that $\hat{R}$ looks similar to $B$.
What we can note quantitatively, and which is of greater interest, is the fact that our more space efficient construction has a \emph{lower} mean and maximum squared error.
We seem to be able to get the best of both worlds: lower error \emph{and} space usage.

This is not a computational fluke.
Simply put, the optimality results in \cite{henzinger_almost_2023, dvijotham2024efficient} show that $B^2=A$ is \emph{asymptotically optimal}, but on fixing $n$ one can hope to improve lower order terms for the error, and, as we show, \emph{even in sublinear space}.
In the particular case of the maximum squared error, $B^2=A$ is exactly optimal over Toeplitz factorizations \cite{dvijotham2024efficient}, implying that our binned factorizations can achieve a lower error than any Toeplitz factorization of $A$.
We go on to show that this fact is demonstrably true up to at least $n=10^4$.
\begin{figure*}[t!]
\centering
\subfigure[Space complexity needed for fixed relative MeanSE vs.\ $n$.]{\includegraphics{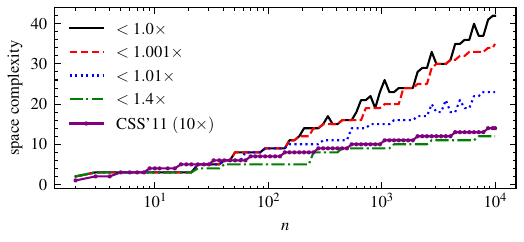}\label{fig:mse_space_vs_approx}}\hfill
\subfigure[Space complexity needed for fixed relative MaxSE vs.\ $n$.]{\includegraphics{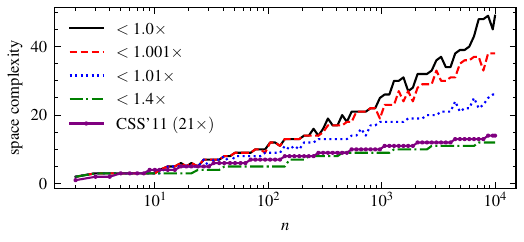}\label{fig:maxse_space_vs_approx}}\
\subfigure[Relative MeanSE vs.\ space complexity for different $n$.]{\includegraphics{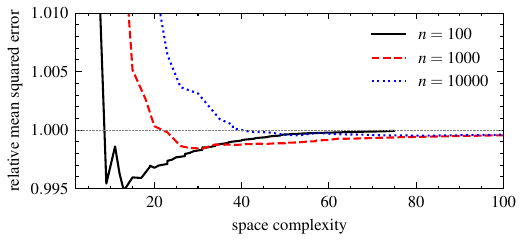}\label{fig:mse_mult_vs_space}}\hfill
\subfigure[Relative MaxSE vs.\ space complexity for different $n$.]{\includegraphics{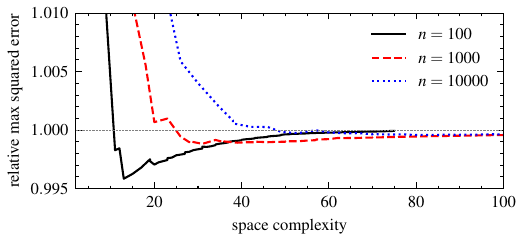}\label{fig:maxse_mult_vs_space}}
\caption{
    Plots showing the trade-off between space complexity and multiplicative blow-up in the mean and maximum squared error for our factorizations relative to~\cite{henzinger_almost_2023}.
    All factorizations shown here were produced by running \Cref{alg:matrix_to_binning} with $c = 1 - 1/d$ for integers $d\geq 2$ and $\tau = 1/n$.
    In \Cref{fig:mse_space_vs_approx} and \Cref{fig:maxse_space_vs_approx}, $d$ was initialized to $2$ for each $n$, and then incremented until the error was sufficiently small for each curve.
    In \Cref{fig:mse_mult_vs_space} and \Cref{fig:maxse_mult_vs_space} the points are generated from using $1.1 \leq d \leq 100$.
    \Cref{fig:mse_space_vs_approx} and \Cref{fig:maxse_space_vs_approx} shows the space complexity (binning size) needed by the approximation $\hat{L}$, as a function of $n$, to be within a given multiple of this error.
    The dashed lines show the space complexity needed to be no more than some small fraction above the error, whereas the black line shows space needed for our method to achieve a \emph{smaller} error.
    The line with dots shows the space complexity of the binary mechanism in \cite{chan_private_2011}, whose error is asymptotically worse by a factor of $10$ or $21$, for mean and maximum squared error respectively.
    \Cref{fig:mse_mult_vs_space} and \Cref{fig:maxse_mult_vs_space} shows the trade-off between the blow-up in error versus space for different values of $n$.
}\label{fig:big_result}
\end{figure*}
The main takeaway from \Cref{fig:big_result} is that our method indeed scales well for the region of $n$ shown.
If a relatively weak assumption on the blow-up in error is acceptable, then the empirical data in Fig.\ \ref{fig:mse_space_vs_approx} and \ref{fig:maxse_space_vs_approx} is compatible with our method running in $O(\log n)$ space.
We also wish to underline the fact that the plots do not change much whether one considers maximum or mean squared error.
While this is also true for our bounds, it is not obvious that it should be carried out in practice given that our analysis is not believed to be tight.

\subsection{Results with momentum and weight decay}\label{sec:empirical_momentum_and_decay}

To showcase the generality of our approach we also demonstrate it for the square-root factorization of the counting matrix $L=R=\sqrt{A_{\alpha, \beta}}$ that incorporates weight decay ($\alpha < 1$) and/or momentum ($\beta > 0$).
First, in \Cref{fig:small_examples_alpha_beta}, we showcase what solutions for small instances look like when we introduce momentum and weight decay.
\begin{figure}[ht]
\centering
\subfigure[Matrix $B_{\alpha, \beta}$.]{\includegraphics{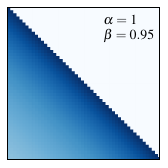}\label{fig:sqrt_alphabeta1}}
\subfigure[Left factor $\hat{L}$.]{\includegraphics{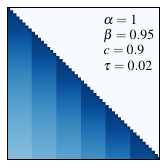}\label{fig:L_alphabeta1}}
\subfigure[Right factor $\hat{R}$.]{\includegraphics{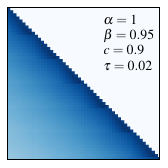}\label{fig:R_alphabeta1}}
\subfigure[Matrix $B_{\alpha, \beta}$.]{\includegraphics{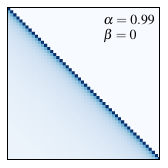}\label{fig:sqrt_alphabeta2}}
\subfigure[Left factor $\hat{L}$.]{\includegraphics{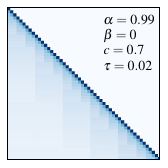}\label{fig:L_alphabeta2}}
\subfigure[Right factor $\hat{R}$.]{\includegraphics{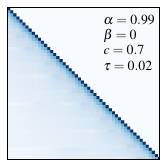}\label{fig:R_alphabeta2}}
\caption{
    Plot of entries of different factorizations of $A_{\alpha, \beta}$ for $n=50$.
    Darker shades imply matrix entries closer to 1; white implies 0.
    \Cref{fig:sqrt_alphabeta1} and \Cref{fig:sqrt_alphabeta2} correspond to the square-root factorization $A_{\alpha, \beta}=B_{\alpha, \beta}^2$.
    $\hat{L}$ is the $\B$-approximation of $B_{\alpha, \beta}$ generated from \Cref{alg:matrix_to_binning} when run with $c$ and $\tau$ as specified in the figures, on $B_{\alpha, \beta}$, and where $\hat{R} = \hat{L}^{-1} A_{\alpha, \beta}$.
    The factorization shown in Fig.\ \ref{fig:L_alphabeta1} and \ref{fig:R_alphabeta1} achieves $\mse(\hat{L}, \hat{R}) / \mse(B_{\alpha, \beta},B_{\alpha, \beta}) = 0.9945$ and $\maxerr(\hat{L}, \hat{R}) / \maxerr(B_{\alpha, \beta},B_{\alpha, \beta}) = 0.9947$, compared to a relative error of $1.015$ and $1.026$ respectively for the factorization in Fig.\ \ref{fig:L_alphabeta2} and \ref{fig:R_alphabeta2}.
    Both our factorizations have a binning size $\abs{\B}$ of $8$.
}\label{fig:small_examples_alpha_beta}
\end{figure}

Qualitatively one can note that introducing momentum ($\beta > 0$) results in a matrix $B_{\alpha, \beta}$ with a mass that is more evenly distributed.
Introducing weight decay ($\alpha < 1$) should intuitively have the opposite effect, where mass becomes more concentrated closer to the diagonal.
For both cases we are able to reduce the space requirement $50$ to $8$ buffers, but we are only able to achieve an improvement in error for the case with momentum.
This trend is replicated in \Cref{fig:big_result_alpha_beta}, where we more systematically investigate the influence of $\alpha$ and $\beta$, and $\alpha < 1$ proves to be the harder case.
\begin{figure*}[t!]
\centering
\subfigure[$\alpha=1, \beta=0.95$.]{\includegraphics{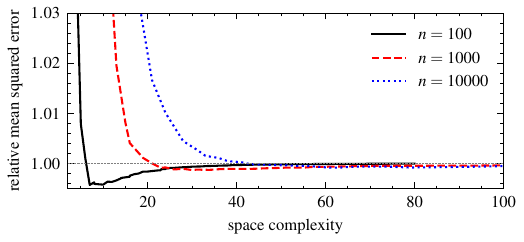}\label{fig:meanse_alphabeta1}}
\subfigure[$\alpha=1, \beta=0.95$.]{\includegraphics{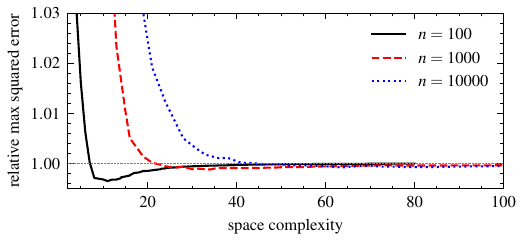}\label{fig:maxse_alphabeta1}}
\subfigure[$\alpha=1, \beta=0.9$.]{\includegraphics{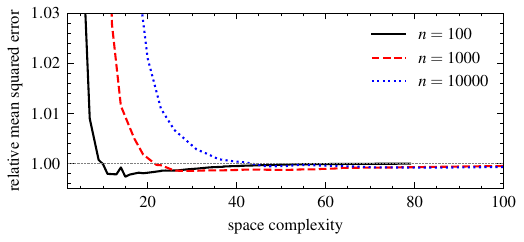}\label{fig:meanse_alphabeta2}}
\subfigure[$\alpha=1, \beta=0.9$.]{\includegraphics{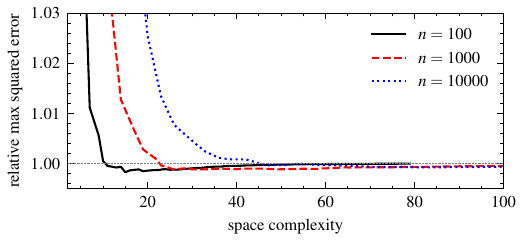}\label{fig:maxse_alphabeta2}}
\subfigure[$\alpha=0.99, \beta=0$.]{\includegraphics{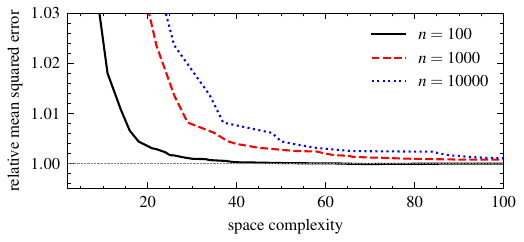}\label{fig:meanse_alphabeta3}}
\subfigure[$\alpha=0.99, \beta=0$.]{\includegraphics{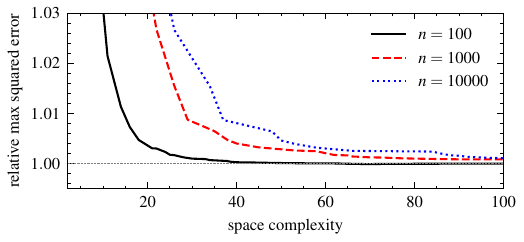}\label{fig:maxse_alphabeta3}}
\subfigure[$\alpha=0.99, \beta=0.95$.]{\includegraphics{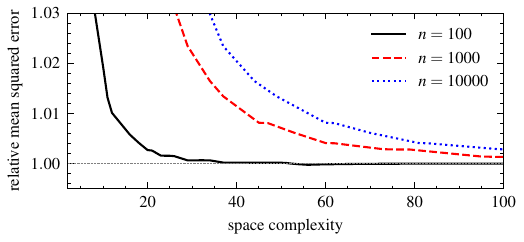}\label{fig:meanse_alphabeta4}}
\subfigure[$\alpha=0.99, \beta=0.95$.]{\includegraphics{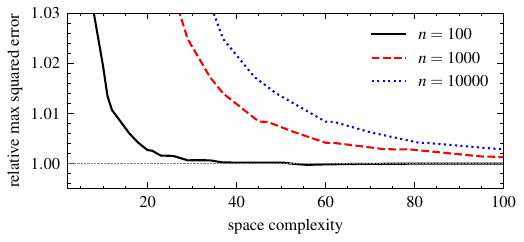}\label{fig:maxse_alphabeta4}}
\caption{
    Plots showing the trade-off between space complexity and multiplicative blow-up in the mean and maximum squared error for our factorizations relative to the square-root factorization of $A_{\alpha, \beta}$, for different $\alpha, \beta$.
    All factorizations shown here were produced by running \Cref{alg:matrix_to_binning} with $c = 1 - 1/d$ for $1.1 \leq d \leq 100$, and $\tau = 1/n$.
}\label{fig:big_result_alpha_beta}
\end{figure*}
The main points to make here is that our technique works well across these parameter choices of $\alpha, \beta$, allowing us to substantially reduce the space at the cost of a constant blow-up in error.
We note that, contrary to the case of the Bennett matrix where $\alpha = 1$, setting $\tau = 1/n$ has an impact when $\alpha < 1$.
Since the entries of $B_{\alpha, \beta}$ decay exponentially away from the diagonal when $\alpha < 1$ (see \Cref{fact:squre_root_factorization}), there can be entries smaller than $\tau = 1/n$, all of which will be grouped into at most two different intervals per row.
As in the case of $\alpha=1, \beta=0$, whether it is maximum or mean squared error has little impact -- our approach is well-suited to both error variants.

\section{Conclusion and discussion}
We have introduced a new framework for approximating factorization mechanisms based on a technique we call \enquote{binning}.
We give an algorithm that always produces a valid binning, and we show that, for a class of factorizations $LR$, calling it with suitable parameters produces a new factorization $\hat{L}\hat{R}$ running in sublinear space at the cost of a small multiplicative increase in error.
This class includes the asymptotically optimal, up to lower order terms, factorization of the counting matrix $L=R=\sqrt{A}$ in \cite{fichtenberger_constant_2023, henzinger_almost_2023, bennett77}.
Empirically we show something stronger.
For $n$ at least up to $10^4$, we produce factorizations with \emph{lower} error \emph{in sublinear space}.
For weaker multiplicative guarantees, our empirical results suggest that our approach leads to logarithmic space complexity.
Beyond this, we also show that our approach works once momentum and weight decay are introduced to the counting matrix.

There are many possible directions for future work.
The main question we leave is the space efficiency possible by binning: our theoretical bounds on space and error do not appear to be tight in general, as demonstrated by our empirical work.
An approach relying on more global properties of the perturbation $P=\hat{L} - L$, rather than bounding each entry, could be a way forward.
Another question worth pursuing is what other linear query matrices (and their factorizations) can be worth studying under binning.
This is in particular relevant to non-Toeplitz matrices, where the techniques in \cite{dvijotham2024efficient} are not applicable.
It would also be of interest to investigate whether the binnings produced by \Cref{alg:matrix_to_binning} can be substantially improved.

Finally, we mention the \emph{bit complexity} of space-efficient mechanisms that rely on real-valued noise.
It seems straightforward to replace Gaussian noise by discrete Gaussian noise~\cite{canonne2020discrete} in our mechanism: Since we only add random samples together there are no numerical issues and we can get high accuracy with only $O(\log n)$ bits per value, assuming that the exact prefix sums can be represented in $O(\log n)$ bits.
For other methods such as \cite{dvijotham2024efficient} it is less clear how many bits of precision are needed to maintain the privacy/utility trade-off.

\medskip
\begin{samepage}
{\bf Acknowledgement.} The authors are affiliated with Basic Algorithms Research Copenhagen (BARC), supported by the VILLUM Foundation grant 54451, and are also supported by Providentia, a Data Science Distinguished Investigator grant from Novo Nordisk Fonden.

\end{samepage}
\balance
\FloatBarrier
\clearpage
\bibliographystyle{IEEEtran}
\bibliography{IEEEabrv,main}

\end{document}